\documentclass[twoside]{article}

\usepackage[preprint]{aistats2026}
\usepackage{graphicx,nicefrac}
\usepackage{multirow}
\usepackage{multicol}
\usepackage{hyperref}
\usepackage{url}
\usepackage{orcidlink}
\usepackage{graphicx}
\usepackage{amsmath}
\usepackage{amssymb}
\usepackage{booktabs}
\usepackage{amsmath}
\usepackage{amsthm}
\usepackage{algorithm}
\usepackage{algpseudocode}
\usepackage{mdframed}
\usepackage{pifont}
\usepackage{subcaption}
\usepackage{color, colortbl}
\usepackage{mathrsfs}
\usepackage{pgfplots}
\usepackage{tikz}
\usepackage{paralist}
\usepackage{thmtools}

\newtheorem{lemma}{Lemma}

\newtheorem{corollary}{Corollary}
\newtheorem{definition}{Definition}
\newtheorem{assumption}{Assumption}

\newtheorem{example}{Example}

\newcommand{\I}{\mathbf{I}}

\newcommand{\Norm}{\mathcal{N}}
\newcommand{\dconv}{\xrightarrow{d}}

\begin{document}

\twocolumn[
\aistatstitle{Beyond Pooling: Matching for Robust Generalization under Data Heterogeneity}

\runningtitle{Beyond Pooling: Matching for Robust Generalization under Data Heterogeneity (AISTATS'26)}

\aistatsauthor{ Ayush Roy \And Rudrasis Chakraborty \And  Lav R.\ Varshney \And Vishnu Suresh Lokhande }

\aistatsaddress{ SUNY Buffalo \And  Lawrence Livermore National Lab \And SUNY Stony Brook \And SUNY Buffalo } 
]

\begin{abstract}
Pooling heterogeneous datasets across domains is a common strategy in representation learning, but naive pooling can amplify distributional asymmetries and yield biased estimators, especially in settings where zero-shot generalization is required. We propose a matching framework that selects samples relative to an adaptive centroid and iteratively refines the representation distribution. The double robustness and the propensity score matching for the inclusion of data domains make matching more robust than naive pooling and uniform subsampling by filtering out the confounding domains (the main cause of heterogeneity). Theoretical and empirical analyses show that, unlike naive pooling or uniform subsampling, matching achieves better results under asymmetric meta-distributions, which are also extended to non-Gaussian and multimodal real-world settings. Most importantly, we show that these improvements translate to zero-shot medical anomaly detection, one of the extreme forms of data heterogeneity and asymmetry. The code is available on \href{https://github.com/AyushRoy2001/Beyond-Pooling}{Github}. \footnote{Paper accepted at Proceedings of the 29th International Conference on Artificial Intelligence and Statistics (AISTATS) 2026, Tangier, Morocco.}
\footnote{Corresponding authors: A. Roy and V.S. Lokhande.}
\end{abstract}

\section{Introduction}
Pooling datasets from multiple institutions presents an attractive approach for studies constrained by limited sample sizes, as is often the case in biomedical research involving human subjects \cite{thompson2014enigma}. The aggregation of data across sites yields larger sample sizes and enhanced statistical power, thereby facilitating more robust analyses. Frequently, biomedical investigations face restrictions in recruiting subjects, particularly when the population of interest is rare, the relevant measurements are challenging to acquire, or when logistical and financial barriers preclude large-scale data collection. Although single-site datasets may suffice to address primary research questions, there is often interest in secondary analyses aimed at uncovering subtle associations between specific predictors and response variables. Such endeavors benefit significantly from pooled datasets, which boost the statistical signal and enable exploration of scientific questions that would otherwise remain inaccessible with smaller cohorts from individual institutions \cite{zhou2018statistical}\cite{lokhande2022equivariance}.

The feasibility of pooling datasets across institutions has been demonstrated in practical settings, supported by a mature body of statistical literature that outlines best practices for such analyses, including covariate matching \cite{mahajan2021domain}\cite{lokhande2022equivariance} and meta-analysis \cite{kim2016quasi}. However, careful attention is required when aggregating data from multiple sources, as increasing training data in these multi-domain scenarios can sometimes lead to diminished overall accuracy \cite{gulrajani2020search,zhou2022domain}, uncertainty in fairness outcomes, and reduced performance for the most challenging subgroups \cite{shen2024data}.

These concerns extend beyond biomedical research, appearing in machine learning and computer vision domains where pooling distinct datasets is common, especially in anomaly detection tasks \cite{huang2024adapting}. Unlike anomalies in natural scene data, which are typically pronounced and readily identifiable, medical anomalies are subtle and often reflect slight deviations within complex imaging modalities such as chest X-rays, MRIs, CT scans, and retinal OCTs. Although comprehensive frameworks exist for mitigating sample selection bias and accounting for differences in population characteristics, there remains a scarcity of adaptive methodologies capable of accommodating sequentially added domains to facilitate robust pooled analyses, an important consideration across a range of scientific applications.

In this work, we present an in-depth theoretical study of pooled dataset construction via matching, focusing on scenarios where domains are sequentially incorporated. We systematically compare naive pooling, uniform subsampling, and matching as strategies for combining data across multiple sites/domains. The proposed matching framework adaptively selects samples relative to an evolving centroid, iteratively refining the representation. We provide rigorous guarantees demonstrating that while all pooling methods converge to the population mean, naive pooling and subsampling retain domain-level heterogeneity, whereas matching uniquely filters out this extra variance and aligns with the target distribution. In finite-domain settings, naive pooling and subsampling are susceptible to unbounded error and bias, while matching delivers explicit finite-sample robustness. Synthetic experiments further validate these theoretical results, showing that matching maintains stability under practical sample sizes, whereas subsampling exhibits amplified errors. Moreover, our findings extend robustly to zero-shot medical anomaly detection, supporting the generalizability of the approach.

\section{Preliminaries of Pooling Strategies}
\label{sec:theory}
Let the target test distribution be $\mathscr{D}_{\mathrm{test}} = \mathcal{N}(\boldsymbol{\mu}_*, \sigma^{2} \mathbf{I}_d)$,
an isotropic Gaussian in $\mathbb{R}^d$.  
Let $K$ Gaussian component distributions be denoted as $\{Q_k\}_{k=1}^K$, where $Q_k = \mathcal{N}(\boldsymbol{\mu}_k, \sigma^{2} \mathbf{I}_d)$, $\mu_k \stackrel{i.i.d.}{\sim} \mathscr{D}_{\mu}$,
with $\mathbb{E}[\boldsymbol{\mu}_k] = \mu_*$ and
$\mathrm{Cov}[\boldsymbol{\mu}_k] = \Sigma_{\mu}$. This setup can be understood hierarchically. Firstly, we sample a domain mean $\boldsymbol{\mu}_k$ from the meta-distribution $\mathscr{D}_{\mu}$. Then, conditioned on $\boldsymbol{\mu}_k$, we define the domain distribution $Q_k = \mathcal{N}(\boldsymbol{\mu}_k, \sigma^2 \mathbf{I}_d)$, from which draw i.i.d.\ samples $\mathbf{x} \sim Q_k$. The complete data-generation process can be expressed as a Bayes-style factorization:
\begin{equation}
\label{eq:intuition}
p(\mathbf{x}) 
= \sum_{k=1}^K p(\mathbf{x} \mid Q_k)\,p(Q_k \mid \boldsymbol{\mu}_k)\,p(\boldsymbol{\mu}_k \mid \mathscr{D}_\mu).
\end{equation}
That is $\boldsymbol{\mu}_k \sim \mathscr{D}_\mu, Q_k \mid \boldsymbol{\mu}_k = \mathcal{N}(\boldsymbol{\mu}_k, \sigma^2 \mathbf{I}_d)$. This introduces two levels of randomness: \begin{inparaenum}[(i)]
    \item domain-level variation from $\mathscr{D}_\mu$ (inter-domain)
    \item sample-level variation within each $Q_k$ (intra-domain)
\end{inparaenum}. The $Q_k$s are not identical, because each is centered at  different $\boldsymbol{\mu}_k$s. Their randomness comes from drawing $\boldsymbol{\mu}_k \sim \mathscr{D}_{\mu}$. Note that although samples $\mathbf{x}_i \sim Q_i$ and $\mathbf{x}_j \sim Q_j$ are i.i.d. conditioned on the domain mean $\boldsymbol{\mu}_i$ and $\boldsymbol{\mu}_j$ respectively, {\it $\mathbf{x}_i$ and $\mathbf{x}_j$ are not exchangeable.}

\textbf{Intuition on Eq. \ref{eq:intuition}:} This factorization simply expresses a sample $x$ is generated following these steps:
i) first a domain mean $\mu_k$ is drawn from the meta-distribution $\mathscr{D}_\mu$
ii) this $\mu_k$ defines a domain $Q_k$ (i.e., the distribution for that domain)
iii) finally, data $x$ is sampled from the distribution of that domain, $p(x \mid Q_k), p(Q_k \mid \mu_k), p(\mu_k \mid \mathscr{D}_\mu)$.

Thus, the meta-distribution $\mathscr{D}_\mu$ is the distribution over domain means $\mu_k$ (e.g., $\mathscr{D}_\mu$ is a medical imaging database of all hospitals, and $\mu_k$ is a table inside $\mathscr{D}_\mu$ corresponding to one hospital). The shape of $\mathscr{D}_\mu$ (e.g., symmetric vs.\ asymmetric, light-tailed vs.\ heavy-tailed) fully determines whether the collection of domains is balanced or skewed. Thus, symmetry or asymmetry arises entirely at the level of sampling $\mu_k$ (e.g., $\mu_1$ being data from a hospital in the USA and $\mu_2$ being data from a hospital in India will have different distributions). The variability across datasets comes from variability in $\mu_k$ (the domain heterogeneity), and each domain then generates its own data.

We formally define three pooling strategies, naive pooling (Def.~\ref{def:naive_pooling}), 
uniform subsampling (Def.~\ref{def:subsampling}), and 
causal matching (Def.~\ref{def:matching}). 
Each corresponds to a different way of aggregating samples across domains.

\begin{definition}[\textbf{Naive Pooling}]
\label{def:naive_pooling}
Naive pooling aggregates all available samples from every domain:
\begin{equation}
\mathcal{U}^{(K)} = \bigcup_{k=1}^K \{\,\mathbf{x}_{k,i} :  \mathbf{x}_{k,i} \sim Q_k,\; i=1,\dots,N_k \,\}.
\end{equation}
Equivalently, the pooled empirical distribution is
$\hat P_{\mathrm{pool}}^{(K)}(\mathbf{x}) 
= \frac{1}{\sum_{k=1}^K N_k} 
\sum_{k=1}^K \sum_{i=1}^{N_k} \delta(\mathbf{x} - \mathbf{x}_{k,i})$,
where $N_k$ is the sample size of domain $k$.
\end{definition}

\begin{definition}[\textbf{Uniform Subsampling}]
\label{def:subsampling}
Uniform subsampling randomly selects domains without bias and then samples a small subset of points from them.
Formally, let $I \subset \{1,\dots,K\}$ be a set of indices chosen uniformly at random.
For each selected domain $Q_k$, we draw $n \ll N_k$ samples (with replacement):
\begin{equation}
\mathcal{U}_{\mathrm{sub}}^{(n)} 
= \bigcup_{k \in I} \{\, \mathbf{x}_{k,j} : \mathbf{x}_{k,j} \sim Q_k,\; j=1,\dots,n \,\}.
\end{equation}
The corresponding empirical distribution is $\hat P_{\mathrm{sub}}^{(n)}(\mathbf{x})
= \frac{1}{|I|\cdot n} \sum_{k \in I}\sum_{j=1}^n \delta(\mathbf{x}) - \mathbf{x}_{k,j})$. Thus subsampling introduces randomness both at the \emph{domain level} (random choice of $Q_k$) and at the \emph{sample level} (random choice of $n$ samples from each selected domain).
\end{definition}

\begin{definition}[\textbf{Matching}]
\label{def:matching}
Matching selectively includes domains based on their proximity to a adaptive centroid.
A domain $Q_k$ is included if its mean $\mu_k$ lies within a $\tau$-radius ball of the matching function M: $Q_k \in \mathcal{S}^{(\tau)} 
\quad \iff \quad M(\boldsymbol{\mu}_k - \mathbf{c}_t) < \tau$.
\end{definition}

If $Q_k$ is included, then \emph{all of its samples} are admitted:
\begin{equation}
\mathcal{U}_{\mathrm{match}}^{(\tau)} 
= \bigcup_{k \in \mathcal{S}^{(\tau)}} \{\, \mathbf{x}_{k,i} : \mathbf{x}_{k,i} \sim Q_k,\; i=1,\dots,N_k \,\}.
\end{equation}
The matched empirical distribution is $\hat P_{\mathrm{match}}^{(\tau)}(\mathbf{x}) 
= \frac{1}{\sum_{k \in \mathcal{S}^{(\tau)}} N_k}
\sum_{k \in \mathcal{S}^{(\tau)}} \sum_{i=1}^{N_k} \delta(\mathbf{x} - \mathbf{x}_{k,i})$.
The centroid is updated iteratively as $\mathbf{c}_{t+1} 
= \frac{1}{\sum_{k \in \mathcal{S}^{(\tau)}} N_k} 
\sum_{k \in \mathcal{S}^{(\tau)}} \sum_{i=1}^{N_k} \mathbf{x}_{k,i}$.
Thus, unlike subsampling, matching operates via a \emph{biased domain inclusion step}, where entire domains are accepted or rejected.

Specifically, matching can be viewed through the lens of domain-level propensity scores \ref{sec:symmetry_def}. 
For each domain $k$, define $W_k = \mathbb{I}\big(M(\boldsymbol{\mu}_k - \mathbf{c}_t) < \tau\big)$, where $W_k=1$ indicates that all samples from domain $Q_k$ are included in the matched set. 
This parallels causal inference, where $W_k$ represents treatment assignment given covariates \cite{motiian2017unified} (here, $\boldsymbol{\mu}_k$). 
The ``treatment'' is domain inclusion, and the covariate is the domain mean. 
Unbiased causal estimation relies on three well-known assumptions: \emph{ignorability}, \emph{positivity}, and \emph{consistency}. 
These map naturally into our framework: \textit{ignorability} requires that inclusion depends only on observable $\boldsymbol{\mu}_k$, \textit{positivity} requires that each relevant domain has a nonzero chance of inclusion, and \textit{consistency} requires that included domains genuinely contribute samples from their $Q_k$ (which holds since intra-domain samples are i.i.d.). 

These assumptions clarify why matching succeeds where naive pooling and subsampling can fail. 
Naive pooling and subsampling, while exchangeable in the probabilistic sense, they do not satisfy ignorability, i.e., do not enforce \emph{conditional balance with the target} $\boldsymbol{\mu}_*$: they freely include domains whose $\boldsymbol{\mu}_k$ deviate systematically from the target, introducing target-irrelevant variation (akin to confounding in observational studies; see Section~\ref{sec:symmetry_def}). 
Matching, by contrast, \emph{leverages} exchangeability and explicitly conditions inclusion on proximity to the adaptive centroid, thereby ensuring \emph{ignorability and overlap with respect to the target}. 
This mechanism yields a form of double robustness: consistency is achieved as long as either \begin{inparaenum}[(i)] \item the inclusion rule (propensity model $W_k$) or \item the outcome model (centroid update) is correctly specified \end{inparaenum}.
Consequently, mild misspecification of the threshold $\tau$ does not compromise limiting behavior, making matching a stable strategy under domain heterogeneity. 

In the subsequent sections, we analyze the aforementioned points theoretically while validating them experimentally. Particularly, Section \ref{sec:infinite_K} shows convergence of all methods for infinte number of domains, Section \ref{sec:symmetric_finite} demonstrates the challenges under finite K. Section \ref{sec:multimodal} and \ref{sec:gaussian_to_real_world} shows the extension of the proposed method to multimodal data and non gaussian settings respectively. All the theoretical insights are validated by the experimental results (Section \ref{sec: Anomaly detection}) with additional explanations and experimental evidence been provided in Appendix.

\section{Asymptotic Convergence Analysis ($K \to \infty$)}
\label{sec:infinite_K}
Here, we show that all the pooling strategies converge to $\boldsymbol{\mu}_*$ when we have infinite domains, i.e., $K \to \infty$. %This is a less realistic scenario where all methods achieve convergence while is necessary to analyze for understanding the challenges under finite-K setting. 

\begin{figure*}[t]
    \centering
    \includegraphics[width=0.9\textwidth]{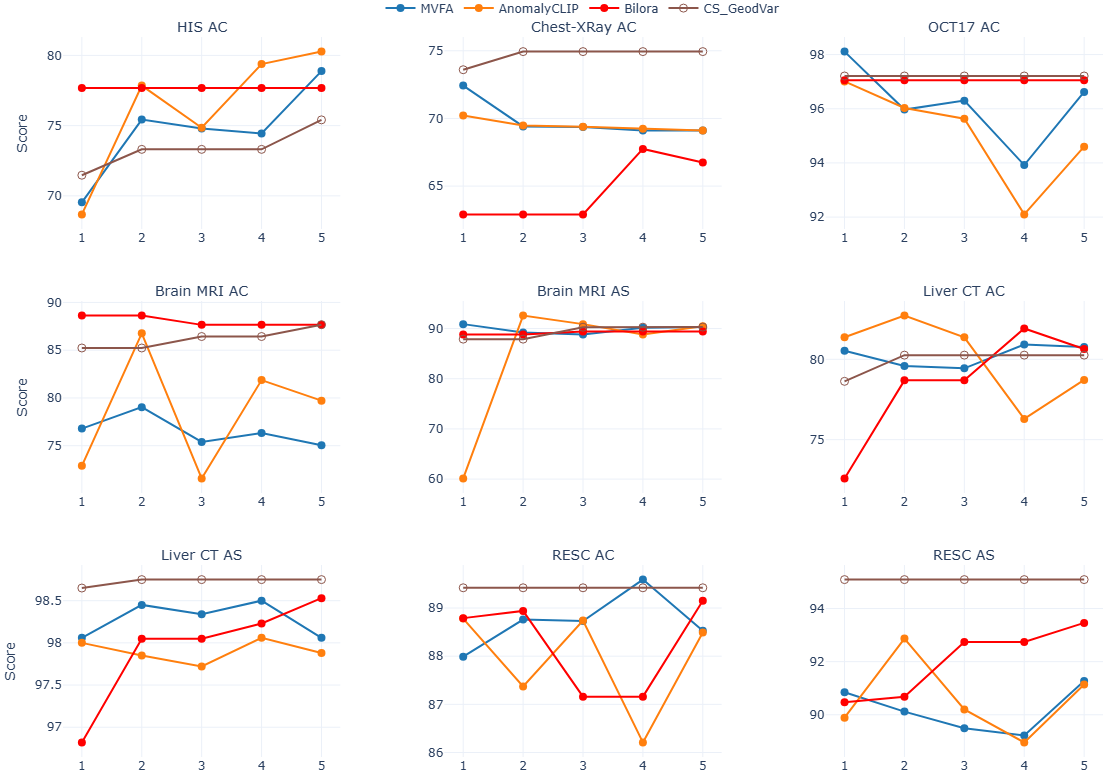}
    \caption{\small Comparison of the proposed method with MVFA \cite{huang2024adapting}, AnomalyCLIP \cite{zhou2023anomalyclip}, and BiLORA \cite{zhu2025bilora}. Performance is reported in terms of domain alignment (DA), anomaly classification (AC), and anomaly segmentation (AS). The x-axis represents the sequential domain addition, and the y-axis represents the AUC scores. Comparison of Domain Alignment (DA) scores across all datasets for the Base, Agnostic, and GeoDVar methods. MVFA \cite{huang2024adapting} shows moderate alignment with DA scores clustered around 2 (e.g., 2.2365 for HIS, 2.1146 for Chest-XRay, 2.0501 for OCT17, 2.0466 for Brain MRI AC, 2.0335 for Liver CT AC). AnomalyCLIP \cite{zhou2023anomalyclip} yields inconsistent and often poorer alignment, with scores varying widely from 1.0102 to 3.2632 across different tasks. BiLORA \cite{zhu2025bilora} achieved better DA scores than MVFA \cite{huang2024adapting} and AnomalyCLIP \cite{zhou2023anomalyclip} (HIS AC-4.0, ChestXray AC-3.145, OCT17 AC-4.0, BrainMRI AC-4.0, BrainMRI AS-4.012, LiverCT AC-3.158, LiverCT AS-4.030, RESC AC-3.081, RESC AS-4.072). In contrast, our method achieves consistently superior domain alignment, with DA scores at or exceeding 4.0 for all datasets and tasks, including peaks of 4.0736 for Brain MRI detection (AC) and 4.0484 for Brain MRI segmentation (AS), demonstrating its robust performance in matching complex domain distributions.}
    \label{fig:final}
\end{figure*}

\begin{restatable}[\textbf{Asymptotic Behavior as $K \to \infty$}]{theorem}{AsymptoticBehavior}
\label{thm:asymptotic_behavior}
Let $\{\boldsymbol{\mu}_k\}_{k=1}^K \overset{i.i.d.}{\sim} \mathscr{D}_{\mu}$ with $\mathbb{E}[\boldsymbol{\mu}_k] = \boldsymbol{\mu}_*$ and $\mathrm{Cov}(\boldsymbol{\mu}_k) = \Sigma_{\mu}$.  
Suppose each domain $Q_k$ generates i.i.d.\ samples $\mathbf{x}_{k,i} \sim \mathcal{N}(\boldsymbol{\mu}_k, \sigma^2 I_d)$, where $\sigma^2 I_d$ is the within-domain covariance. Then, as $K \to \infty$:
    a) \textbf{Naive pooling:} $\hat P_{\mathrm{pool}}^{(K)} \;\;\xrightarrow{d}\;\; \mathcal{N}\!\bigl(\boldsymbol{\mu}_*, \, \sigma^2 I_d + \Sigma_{\mu}\bigr)$.
    b) \textbf{Uniform subsampling (fixed $n$):} $\hat P_{\mathrm{sub}}^{(n)} \;\;\xrightarrow{d}\;\; \mathcal{N}\!\bigl(\boldsymbol{\mu}_*, \, \sigma^2 I_d + \Sigma_{\mu}\bigr)$.
    c) \textbf{Matching (fixed $\tau > 0$):} 
    If the iterative centroid satisfies $\mathbf{c}_n \xrightarrow{p} \mu_*$, then $\hat P_{\mathrm{match}}^{(\tau)} \;\;\xrightarrow{d}\;\; \mathcal{N}\!\bigl(\boldsymbol{\mu}_*, \, \sigma^2 I_d\bigr)$.
\end{restatable}

\textit{Takeaway:} From Theorem \ref{thm:asymptotic_behavior} (see Supplementary Sec. \ref{sec:proof_asymptotic_behavior} for proof), we see that as $K \to \infty$, all pooling methods converge to distributions centered at $\boldsymbol{\mu}_*$, but with critical differences in their covariance structures. naive pooling and uniform subsampling converge to $\mathcal{N}(\boldsymbol{\mu}_*, \sigma^2 \I_d + \Sigma_{\mu})$, retaining the inter-domain variance $\Sigma_{\mu}$ that reflects domain heterogeneity. In contrast, matching converges to $\mathcal{N}(\boldsymbol{\mu}_*, \sigma^2 \I_d)$, effectively filtering out the inter-domain variance and matching the target test distribution exactly. This shows  matching's unique ability to eliminate domain-level heterogeneity while preserving within-domain variability.

Corollary~\ref{cor:exchangeability} (Supplementary Sec.~\ref{sec:cor_exchangeability_proof}) shows that all pooling strategies (naive pooling, subsampling, and matching) are exchangeable because they are symmetric functions of the i.i.d.\ domain draws $\{\mu_k\}$. 
This holds for finite $K$ and as $K \to \infty$: permuting domain indices does not change the joint distribution of the pooled estimator. 
Exchangeability implies that, conditional on $\{\boldsymbol{\mu}_k\}$, domain inclusion is “as if randomized,” which corresponds to the \emph{ignorability} condition used by propensity-score methods \cite{rosenbaum1983central}. The distinction lies in how exchangeability is \emph{used}. Naive pooling and subsampling do not enforce balance with respect to the target $\mu_*$ and thus can suffer finite-sample bias under asymmetric $\mathscr{D}_\mu$.  Matching explicitly exploits exchangeability within a propensity framework: by conditioning inclusion on distance to the centroid, it secures \emph{ignorability and overlap relative to the target}, aligning domain inclusion with the causal estimand of interest. Thus, while all strategies are exchangeable, only matching operationalizes this property to mitigate finite-sample bias. However, $K \to \infty$ is not a practical setting. In the subsequent sections we will analyze the unbiasness/biasness of the pooling techniques under finite K (for symmetric and assymetric of meta distribution $\mathscr{D}_{\mu}$ as seen in Def.~\ref{def:sym_dif}).

\section{Convergence Analysis under Symmetric $\mathscr{D}_{\mu}$ for $K$}
\label{sec:symmetric_finite}

\subsection{Symmetry of Meta Distribution $\mathscr{D}_{\mu}$}
\label{sec:symmetry_def}
\begin{definition}[\textbf{Distributional Symmetry}]
\label{def:sym_dif}
A meta-distribution $\mathscr{D}_{\mu}$ is \emph{symmetric} around $\boldsymbol{\mu}_*$ if $\boldsymbol{\mu}_k - \boldsymbol{\mu}_* \overset{d}{=} \boldsymbol{\mu}_* - \boldsymbol{\mu}_k$.
\end{definition}

Intuitively, symmetry means domain means are equally likely to deviate above or below the population center, while on the contrary, asymmetry implies systematic bias in one direction. Furthermore, Assumption \ref{ass:moment_tail} highlights the relevance of $\mathscr{D}_{\boldsymbol{\mu}}$ for real world scenarios where sampling a domain ($Q_k$) that is far away from the mean $\boldsymbol{\mu}_*$ is highly likely and not a rare event. In example \ref{eg:example_assymetry}, the occasional $\boldsymbol{\mu}_*+5$ domains are averaged out by the majority of $\boldsymbol{\mu}_*$ domains as $K$ grows. Hence, in the infinite-sample limit, deviations vanish and the pooled mean coincides with $\boldsymbol{\mu}_*$.

\begin{example}
\label{eg:example_assymetry}
Consider a one-dimensional meta-distribution $\mathscr{D}_{\mu}$ defined as follows: $\boldsymbol{\mu}_k = 
\boldsymbol{\mu}_* + c$ with probability p, $\boldsymbol{\mu}_*$ with probability 1-p.
This distribution satisfies Assumption~\ref{ass:moment_tail}: it has a finite mean and variance, sub-exponential tails, finite fourth moment, and importantly nontrivial directional mass at $\boldsymbol{\mu}_*+c$. 

\emph{Finite-$K$ sampling:} If $K$ domains are drawn, the probability that at least one $Q_k$ is centered at $\boldsymbol{\mu}_*+c$ is $1 - (1-p)^{K} $ ($p$ can be considered small based on the fact that as we go far from the mean $\mu_*$, the probability decreases). Thus, with high probability we obtain at least one ``outlier’’ domain whose mean is far from $\boldsymbol{\mu}_*$. This illustrates that asymmetric realizations are not rare but \emph{expected} under the assumption. \emph{As $K\to\infty$:} The law of large numbers guarantees that $\frac{1}{K}\sum_{k=1}^K \boldsymbol{\mu}_k \;\;\longrightarrow\;\; \boldsymbol{\mu}_*$ almost surely.
\end{example}

\subsection{Unbiased Estimation for Finite K}
\label{sec:finite_k_symmetry_convergence}

\begin{restatable}[\textbf{Finite-$K$ Unbiasedness under Symmetry}]{theorem}{SymmetricConvergence}
\label{thm:symmetric_convergence}
Suppose the meta-distribution $\mathscr{D}_{\mu}$ is symmetric about $\boldsymbol{\mu}_*$ in the sense of Def.~\ref{def:sym_dif}, so that $\mathbb{E}[\boldsymbol{\mu}_k] = \boldsymbol{\mu}_*$. For any finite number of domains $K \geq 1$:

1) \textbf{Naive Pooling:} Let $\bar{\mu}_K = \frac{\sum_{k=1}^K\sum_{i=1}^{N_k} \mathbf{x}_{k,i}}{\sum_{k=1}^K N_k}, 
\quad x_{k,i}\sim Q_k = \mathcal{N}(\boldsymbol{\mu}_k,\sigma^2 I_d)$. Then $\mathbb{E}[\bar{\mu}_K] = \boldsymbol{\mu}*$. b) \textbf{Uniform Subsampling:} For a uniformly chosen subset $I \subset \{1,\dots,K\}$ and $n$ samples per selected domain, the subsample mean $\bar{\boldsymbol{\mu}}_I = \nicefrac{1}{n|I|}\sum_{k\in I}\sum_{j=1}^n \mathbf{x}_{k,j}$ 
satisfies $\mathbb{E}[\bar\mu_I] = \boldsymbol{\mu}_*$. c) \textbf{Matching:} If the centroid is initialized at the target, $\mathbf{c}^{(0)}=\boldsymbol{\mu}_*$, then at every iteration the matched mean 
$\bar{\mu}_{\mathcal{S}} = \frac{1}{\sum_{k\in \mathcal{S}^{(\tau)}} N_k}\sum_{k\in \mathcal{S}^{(\tau)}}\sum_{i=1}^{N_k}\mathbf{x}_{k,i}$ 
remains unbiased, i.e. $\mathbb{E}[\bar{\mu}_{\mathcal{S}}]=\boldsymbol{\mu}_*$.
\end{restatable}

\textit{Takeaway:} Theorem \ref{thm:symmetric_convergence} (see Supplementary Sec. \ref{sec:thm_symmetric_convergence} for proof.) demonstrates that under symmetric $\mathscr{D}_{\boldsymbol{\mu}}$ (Def.~\ref{def:sym_dif}), all three pooling strategies yield unbiased estimators of the target mean $\boldsymbol{\mu}_*$ for any finite $K$. This establishes a baseline where domain heterogeneity is balanced around the target, allowing all methods to perform well. Furtherore, for matching even if $\mathbf{c}^{(0)} \neq \boldsymbol{\mu}_*$, appropriate selection of $\tau$ ensures that relevant domains are selected and $\bar{\mu}_{\mathcal{S}}$ converges to $\boldsymbol{\mu}_*$.

This is often rare for real world datasets, thus in the next section we analyze the convergence rates of the pooling methods under asymmetric $\mathscr{D}_{\boldsymbol{\mu}}$ (Def.~\ref{def:sym_dif}).

\section{Convergence Analysis under Asymmetric $\mathscr{D}_{\mu}$ for Finite $K$}
\label{sec:data_addition}

While Theorem~\ref{thm:asymptotic_behavior} establishes asymptotic convergence properties and Theorem~\ref{thm:symmetric_convergence} shows finite-K convergence under ideal symmetry, real-world applications involve finite $K$ with distributional asymmetries that create significant challenges. This section analyzes the critical transition from $K$ to $K+1$ domains under asymmetric conditions.

\begin{restatable}[\textbf{Finite-Sample Robustness under Domain Addition}]{theorem}{FiniteRobustness}
\label{thm:finite_robustness}
Let $\bar{\boldsymbol{\mu}}_K := \tfrac{1}{K}\sum_{k=1}^K \boldsymbol{\mu}_k$ denote the empirical average of the first $K$ domain means (with $\mu_k = \mathbb{E}_{x \sim Q_k}[f(x)]$), and let $\varepsilon_K := \|\bar{\boldsymbol{\mu}}_K - \boldsymbol{\mu}^\ast\|_2$. Fix $\delta\in(0,1)$. Under Assumption~\ref{ass:moment_tail}: 1. \textbf{Naive pooling.} With probability at least $1-\delta$, $\varepsilon_{K+1} \;\ge\; \frac{K}{K+1}\,\varepsilon_K \;-\; \frac{\beta\log(2/\delta)}{K+1} \;-\; O\!\left(\sqrt{\tfrac{\log(1/\delta)}{K}}\right)$. 
Moreover,
% by the directional-mass clause of Assumption~\ref{ass:moment_tail},
there exist adversarial $\mu_{K+1}$ such that with constant probability, $\varepsilon_{K+1} \ge \varepsilon_K + \Theta(1)$ (unbounded per-step deterioration).
2. \textbf{Uniform subsampling.} If $m$ domains are drawn uniformly at random and $n$ samples per domain are averaged, then $\mathbb{E}[\varepsilon_{K+1}^2] \;\le\; \varepsilon_K^2 \;+\; \frac{\lambda_{\max}(\Sigma_\mu)}{m} \;+\; \frac{d\sigma^2}{mn}$. Hence $\mathbb E[\varepsilon_{K+1}] \le \varepsilon_K + O\!\big(\tfrac{1}{\sqrt{m}} \vee \tfrac{1}{\sqrt{mn}}\big)$, while worst-case deterioration per step is $\Omega(1/m)$.
3. \textbf{Matching.} Let $S_K$ be the matched set of size $|S_K|$, with centroid $\hat c_K$. Suppose $|S_K|\ge C_0\log(1/\delta)$. If the domain $K+1$ is excluded, then $\varepsilon_{K+1} \le \varepsilon_K + C\sqrt{\tfrac{\log(1/\delta)}{|S_K|}}$. 
Furthermore, if it is included, then with probability $1-\delta$, $\varepsilon_{K+1} \le \varepsilon_K + \frac{\tau}{m_M(|S_K|+1)} + C\sqrt{\tfrac{\log(1/\delta)}{|S_K|}}$. If $M(\boldsymbol{\mu}_{K+1}-\boldsymbol{\mu}s^\ast)\le \varepsilon_K + \tau/L_M$, then $\Pr(\varepsilon_{K+1}\le\varepsilon_K)\ge 1-\exp(-\Omega(|S_K|))$ (high-probability non-deterioration).
\end{restatable}

\textit{Takeaway:} From Theorem \ref{thm:finite_robustness} (see Supplementary \ref{sec:thm_finite_robustness} for proof) we get he comparison in Table~\ref{tab:theoretical_comparison} (see Supplementary \ref{sec:thm_finite_robustness}) highlights the
fundamental difference between matching and the other strategies. For
\textbf{naive pooling}, the error can jump by a \(\Theta(1)\) amount even as
\(K\) grows, meaning deterioration never vanishes. For \textbf{subsampling}, the
worst-case increase is \(\Omega(1/m)\), which shrinks only with the number of
domains subsampled per step. \textbf{Matching} limits the
worst-case increase to \(O(1/|S_K|)\), where \(|S_K|\) is the size of the
matched set. Since \(|S_K|\) typically grows with \(K\), perturbation
vanishes as more domains are included. Thus, matching converts constant or
\(1/m\)-level risks into a self-shrinking $1/|S_K|$ bound (only
strategy with explicit non-deterioration guarantees  under heterogeneity).

Corollary~\ref{cor:optimal_tau} (see Supplementary Sec. \ref{sec:proof_optimal_tau}) provides a practical criterion for selecting the matching threshold $\tau$ to ensure safe domain addition. By setting $\tau$ appropriately relative to the current error $\epsilon_K$ and the distance of harmful domains under metric $M$, degradation is not observed under the introduciton of new domains. Table~\ref{tab:theoretical_comparison} summarizes these differences: only causal matching simultaneously provides bounded deterioration, explicit outlier robustness, and a non-deterioration safety guarantee. In the following section, we see that these theoretical guarantees of matching and the properties of the matched set $S_K$ (see Supplementary \ref{sec:matched_properties}) translate to arbitrary distribution.

\textbf{Synthetic experiments:} \label{sec:synthetic_experiments} We conducted controlled synthetic experiments to empirically validate the guarantees presented in Theorems~\ref{thm:asymptotic_behavior}, \ref{thm:symmetric_convergence}, and \ref{thm:finite_robustness} under enhanced challenges. Each domain was generated as $Q_k = \mathcal{N}(\mu_k, \sigma^2 I_d)$ with means $\mu_k \sim \mathscr{D}_{\mu}$, while the target was fixed as $\mathscr{D}_{\mathrm{test}} = \mathcal{N}(\mathbf{0}, \sigma^2 I_d)$. All the analysis done so far assumed large per-domain sample size $N$, where $\hat\mu_k \to \mu_k$. In contrast, we used small finite $N$ so that sample noise persists. This revealed that all three pooling strategies degrade compared to large-$N$ theory, but to different extents \cite{nazarovs2021graph}. Pooling and matching remain relatively stable, whereas subsampling suffers the largest errors due to $\sqrt{d/N}$ variance amplification. Experiments targeted the three regimes: (a) \textbf{Asymptotic behavior} (Theorem~\ref{thm:asymptotic_behavior}): convergence as $K \to \infty$ under asymmetric domains. (b) \textbf{Symmetric finite-$K$ convergence} (Theorem~\ref{thm:symmetric_convergence}): error bounds with symmetric domain means. (c) \textbf{Finite-sample robustness} (Theorem~\ref{thm:finite_robustness}): stability when incrementally adding domains under asymmetry. All strategies were implemented as defined in Sec.~\ref{sec:theory}, repeated over 10 seeds and averaged. Matching used iterative centroid refinement with $L_2$ metric, robust median initialization, and convergence criterion $|\mathbf{c}_{t+1}-\mathbf{c}_t|_2 < 10^{-4}$. Additional details can be seen in Supplementary~\ref{sec:synthetic} due to space constrains.

\section{Using Normalizing Flow to Generalize for Arbitrary Distributions}
\label{sec:gaussian_to_real_world}
Although analysis of matching for Gaussian distribution ensures theoretical validity, extended theoretical guarantee to arbitrary distributions would demonstrate the superiority of matching over naive pooling, in line with transport-regularized schemes in multi-marginal settings \cite{mehta2023efficient}. Theorem \ref{thm:flow_bridge} (see Supplementary \ref{sec:thm_flow_bridge} for proof) shows the preservation of the properties (see Supplementary \ref{sec:matched_properties}) of $\mathcal{S}$ under arbitrary distributional settings.

\begin{restatable}[\textbf{Normalizing Flow Transport Theorem with Lipschitz Guarantees}]{theorem}{FlowBridge}
\label{thm:flow_bridge}
Let $p_{\text{data}}$ be a compactly supported data distribution and $p_Z = \mathcal{N}(0, I_d)$. 
Under Assumption~\ref{assump:metric}, there exists a bijective normalizing flow $T: \mathbb{R}^d \to \mathbb{R}^d$ such that:
1) \textbf{Universal Approximation:} For any $\epsilon>0$, $D_{\text{KL}}(T_\# p_Z \,\|\, p_{\text{data}}) < \epsilon$, where $T_\# p_Z$ denotes the \emph{pushforward} of $p_Z$ by $T$ (i.e., the distribution of $T(z)$ for $z \sim p_Z$), and $D_{\text{KL}}$ is the Kullback--Leibler divergence.
2) \textbf{Lipschitz Control:} $T$ can be constructed with finite Lipschitz constant $L_T$ controlled by architectural norms; in particular, if $T = T_L \circ \cdots \circ T_1$ and the Jacobian of layer l, $\|J_{T_\ell}\|_{\mathrm{op}} \le C_\ell$ ($C_l$ represents the maximum operator norm of the Jacobian for the l-th layer), then $L_T \le \prod_{\ell=1}^L C_\ell$.
3) \textbf{Variance Preservation:} For any matched set $\mathcal{S}_Z = \{z : \|z-\mathbf{c}_z\|_2 < \tau\}$, letting $\mathcal{S}_X = T(\mathcal{S}_Z)$, $\frac{1}{|\mathcal{S}_X|} \sum_{x \in \mathcal{S}_X} \|x - T(\mathbf{c}_z)\|_2^2 \le L_T^2 \tau^2$.
4) \textbf{Concentration Guarantee:} With $\bar{x} = \frac{1}{|\mathcal{S}_X|}\sum_{x \in \mathcal{S}_X} x$, we have $\|\bar{x} - T(\mathbf{c}_z)\|_2 \le L_T \tau$.
\end{restatable}

\textit{Takeaway:} Theorem \ref{thm:flow_bridge} (see Supplementary Sec. \ref{sec:thm_flow_bridge} for proof) demonstrates that the favorable properties of matching extend from Gaussian to real-world non-Gaussian data through normalizing flows. The utilization of normalizing flow is motivated by the invariance of Bayes error under invertible transformations, as established by \cite{theisen2021evaluating}. This allows us to employ normalizing flows as approximately invertible mappings, ensuring that the robustness bounds of Theorem~\ref{thm:finite_robustness} remain valid beyond Gaussian distributions. The key insight is that flows preserve the geometric structure of matched sets: tight clusters in the latent space ($\|z - \mathbf{c}_z\|_2 < \tau$) map to tight clusters in the data space (variance $\leq L_T^2 \tau^2$). This means the theoretical guarantees for Gaussian data (bounded error, variance reduction, and concentration) carry over to real-world distributions. The Lipschitz constant $L_T$ acts as a distortion factor, showing that well-behaved flows (small $L_T$) maintain the matching benefits. This bridges our Gaussian analysis to practical applications where data exhibit complex, multimodal structure.

\section{Extension to Multimodal Settings}
\label{sec:multimodal}

The theoretical framework established thus far assumes a unimodal distribution where a single centroid $\mathbf{c}_n$ converges to the target mean $\boldsymbol{\mu}_*$. However, real-world applications, particularly in medical imaging, often involve multimodal distributions corresponding to multiple classes of the dataset. We define multimodal data distribution as in Def.~\ref{def:multimodal}.

\begin{definition}[Multimodal Data Distribution]
\label{def:multimodal}
A data distribution $P$ is \emph{M-modal} if it can be expressed as a mixture of $M$ unimodal components $P(x) = \sum_{m=1}^M \pi_m P_m(x)$, where:
1) $\pi_m \geq 0$ are mixing coefficients with $\sum_{m=1}^M \pi_m = 1$
2) Each $P_m$ is unimodal with mode $\boldsymbol{\mu}_m^*$ and covariance $\Sigma_m$
3) The component distributions may have different shapes and scales.

The target distribution in multimodal settings becomes $\mathscr{D}_{\mathrm{test}} = \sum_{m=1}^M \pi_m \Norm(\boldsymbol{\mu}_m^*, \sigma_m^2 \I_d)$.
\end{definition}

We maintain centroids $\{\mathbf{c}_{n,m}\}_{m=1}^M$ and radii $\{\tau_m\}_{m=1}^M$, i.e., a sample $\mathbf{x}$ is assigned to mode $m$ if $M(\mathbf{x}-\mathbf{c}_{n,m})<\tau_m$ and $M(\mathbf{x}-\mathbf{c}_{n,j})\ge \tau_j$ for all $j\neq m$. Centroids are updated per mode using only assigned samples. Lemma \ref{lemma:multimodal_reduction} (Supplementary Sec. \ref{sec:proof_multimodal_reduction}) ensures M-modal convergence as M unimodal convergence under sufficient separation of the M modes. Under the conditions of Lemma~\ref{lemma:multimodal_reduction}, matching guarantees that adding new domains never increases the error for any mode $\epsilon_{K+1,m} \leq \epsilon_{K,m} \quad \forall m = 1,\ldots,M$ with high probability, where $\epsilon_{K,m} = \|\mathbf{c}_{n,m}^{(K)} - \boldsymbol{\mu}_m^*\|_2$. This follows directly from applying Theorem~\ref{thm:finite_robustness} to each mode independently. The separation condition ensures that harmful domains affecting one mode do not interfere with other modes. Each mode's matching process evolves independently and enjoys the same non-deterioration guarantees as the unimodal case.

\section{Experimental Results for Zero-shot Anomaly Detection} 
\label{sec: Anomaly detection}

\begin{table*}[!t]
\scriptsize
\centering
\caption{\textbf{Comparisons with state-of-the-art zero-shot anomaly detection methods.} The AUCs (in \%) for AC and AS are reported. The best result is in bold, and the second-best result is underlined. * are the results we achieved after reimplementation.}
\resizebox{0.9\textwidth}{!}
{\begin{tabular}{@{}lccccccccc@{}}
\toprule
Method & HIS & ChestXray & OCT17 & \multicolumn{2}{c}{BrainMRI} & \multicolumn{2}{c}{LiverCT} & \multicolumn{2}{c}{RESC} \\  
\cmidrule(lr){2-2} \cmidrule(lr){3-3} \cmidrule(lr){4-4} \cmidrule(lr){5-6} \cmidrule(lr){7-8} \cmidrule(lr){9-10}  
& AC & AC & AC & AC & AS & AC & AS & AC & AS \\  
\midrule  
WinCLIP\cite{jeong2023winclip} & 69.85 & \underline{70.86} & 46.64 & 66.49 & 85.99 & 64.20 & 96.20 & 42.51 & 80.56 \\  
APRIL-GAN\cite{chen2023april} & 72.36 & 57.49 & 92.61 & 76.43 & \textbf{91.79} & 70.57 & 97.05 & 75.67 & 85.23 \\  
AnomalyCLIP*\cite{deng2022anomaly} & \textbf{80.27} & 69.11& 94.60 & 79.72 & 90.47 & 78.72 & 97.88 & 88.49 & 91.14 \\
AdaCLIP\cite{cao2024adaclip} & - & - & 68.8 & 73.1 & - & - & - & - & - \\
Mao et al.\cite{mao2025beyond} & - & - & 91.2 & 73.7 & - & - & - & - & - \\
MVFA-AD*\cite{huang2024adapting} & \underline{78.88} & 69.11 & 96.62 & 75.05 & \underline{90.33} & \textbf{80.77} & 98.06 & 88.53 & 91.27 \\
BiLORA*\cite{zhu2025bilora} & 77.68 & 66.75 & \underline{97.05} & \underline{87.65} & 89.40 & \underline{80.64} & \underline{98.53} & \underline{89.15} & \underline{93.46} \\
\midrule
Ours & 75.41 & \textbf{74.94} & \textbf{97.21} & \textbf{87.67} & 90.29 & 80.26 & \textbf{98.75} & \textbf{89.42} & \textbf{95.09} \\
\bottomrule  
\end{tabular}}  
\label{tab:zero-shot} 
\vspace{-1em}
\end{table*}

The synthetic Gaussian experiments in Section~\ref{sec:synthetic} confirmed our theoretical guarantees while exposing finite per-domain sample effects. Unlike the large-$N$ assumptions in theory, limited $N$ introduces sample-level noise, producing severe degradation under subsampling and milder but still visible declines for pooling and matching. This connects theory with practice: robust adaptation must handle finite-$N$ asymmetries rather than relying solely on asymptotic guarantees.

Zero-shot medical anomaly detection provides the strongest stress test of this setting, extending beyond synthetic domains to one of the most challenging real-world scenarios for distributional generalization. Unlike natural scene anomalies (e.g., surface defects) that are structurally consistent and visually distinct, medical anomalies are subtle, heterogeneous, and often overlap with normal anatomical variation. This induces three extreme forms of heterogeneity: (a) \textbf{modality heterogeneity}, as imaging technologies (X-ray, MRI, CT, OCT) generate fundamentally distinct feature distributions; (b) \textbf{pathological similarity}, as abnormal and normal samples exhibit high feature overlap, making separation difficult; and (c) \textbf{institutional bias}, where scanner protocols and acquisition pipelines introduce asymmetric shifts even within the same modality. Together, this “triple heterogeneity” defines the exact worst-case conditions for naive pooling, breaking exchangeability assumptions and amplifying asymmetry. The zero-shot requirement—generalization to unseen modalities and institutions. further exacerbates these challenges.

Under such conditions, three failure modes emerge when $\mathcal{D}_{\text{test}}$ is unseen: (1) the meta-distribution $\mathscr{D}_\mu$ is inherently asymmetric (Def.~\ref{def:sym_dif}); (2) finite-sample bias becomes unavoidable under pooling; and (3) asymmetry compounds as domains are sequentially added. This is precisely where causal matching is critical. Unlike domain-level pooling, our implementation operates at the \emph{sample level}, using an adaptive threshold $\tau$ in the learned feature space. The feature space filters nuisance variability while preserving anomaly–normal semantics, making sub-exponential sample-level behavior more realistic. Consequently, causal matching provides robustness in the finite-$N$ regime, yielding empirical non-deterioration exactly where pooling and subsampling fail, and extending theoretical guarantees (positivity, ignorability, consistency) beyond the asymptotic, domain-level analysis.

\textbf{Setup:} The baselines and related works are described in Supplementary Sec. \ref{sec:related_works}. We employ the BMAD benchmark \cite{bao2024bmad} spanning five medical domains with six datasets: brain MRI (BrainMRI), liver CT (LiverCT), retinal OCT (OCT2017)), chest X-ray (ChestXray), and digital histopathology (HIS). This selection captures the extreme heterogeneity discussed in Section~\ref{sec: Anomaly detection}, with BrainMRI, LiverCT, and RESC supporting both anomaly classification (AC) and segmentation (AS), while OCT17, ChestXray, and HIS are used for AC only. Following the leave-one-out strategy \cite{huang2024adapting}, we train on datasets ${D_1, D_2, ..., D_{i-1}}$ and evaluate on unseen $D_i$. This setup directly instantiates the theoretical challenges of asymmetric domain adaptation: test domains may exhibit selective semantic alignment with specific training domains, creating the perfect conditions for the performance degradation predicted by Theorem~\ref{thm:finite_robustness}. While selectively fine-tuning on aligned domains might seem appealing, medical data scarcity necessitates using all available data making robust pooling strategies essential.

\textbf{Evaluation Metrics:} We report standard Area Under the ROC Curve (AUC) metrics: image-level AUC for anomaly classification and pixel-level AUC for segmentation. However, these conventional metrics fail to capture a critical aspect of real-world deployment: the \emph{monotonicity of improvement} when incorporating additional domains. Theorem~\ref{thm:finite_robustness} highlighted that under asymmetry, the critical quantity is the per-step deterioration $\varepsilon_{K+1}-\varepsilon_K$, which can be unbounded for naive pooling but is controlled under structured strategies. To empirically measure this effect, we introduce the \textbf{Data Addition Score (DA)}, which translates the theoretical notion of stepwise deterioration into an evaluation metric at the performance level. Formally, given performance sequence $\mathbf{y} = [y_1, y_2, y_3, y_4, y_5]$ where $y_i$ denotes performance after incorporating $i$ domains, and importance weights $\mathbf{s} = [0.1, 0.2, 0.3, 0.4]$ emphasizing later stages, the DA is defined as:
\begin{equation}
\vspace{-1.5em}
\label{eq:15}
\text{DA}(\mathbf{y}) = \sum_{i=1}^{4} \mathbb{1}{\{y_{i+1} \geq y_i\}} \cdot \left( 1 + \frac{y_{i+1} - y_i}{10} \cdot s_i \right).
\end{equation}
Here, the indicator $\mathbb{1}{\{y_{i+1} \geq y_i\}}$ captures the non-deterioration requirement—if performance drops, the score for that step is zero—while the weighted improvement term rewards consistent gains. In this way, the DA directly operationalizes the deterioration analysis of Theorem~\ref{thm:finite_robustness}, providing an empirical lens on a method's improvements as domains are added.

\begin{figure*}[t]
    \centering
    \includegraphics[width=0.9\textwidth]{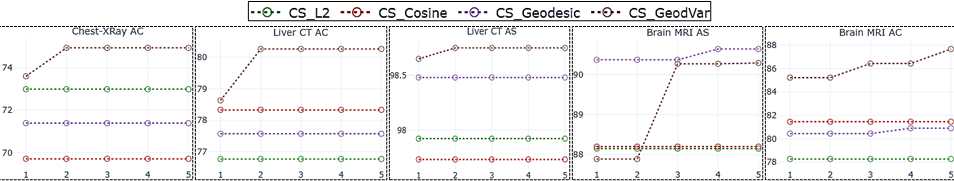}
    \caption{\textbf{Ablation comparing different matching metrics $M$:} Euclidean distance (CS\_L2), cosine similarity (CS\_Cosine), geodesic distance (CS\_Geodesic), and geodesic distance with VACA (CS\_GeodVar). Performance is reported in terms of domain alignment (DA), anomaly classification (AC), and anomaly segmentation (AS). X axis represents the sequential domain addition and the Y axis represents the AUC scores. CS\_L2 and CS\_Cosine achieves DA of 4 for BrainMRI, LiverCT, and ChestXRay. CS\_Geo achieves DA of 4 for ChestXRay and LiverCT while achieving DA of 4.0138 for BrainMRI detection (AC) and 4.0081 for BrainMRI segmentation (AS). CS\_GeodVar achieves DA of 4.0134 for ChestXRay, 4.0163 for LiverCT detection (AC), 4.0010 for LiverCT segmentation (AS), 4.0736 for BrainMRI detection (AC) and 4.0484 for BrainMRI segmentation (AS).}
    \label{fig:ablation}
    \vspace{-1em}
\end{figure*}

\subsection{Matching for Anomaly Detection}
\label{sec:method}
\begin{figure}[t]
    \centering
    \includegraphics[width=0.4\textwidth]{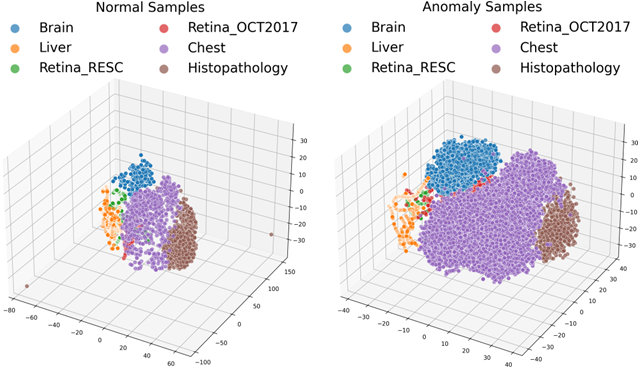}
    \caption{\textbf{Illustration of modality-induced clustering in CLIP feature space.} Embeddings of distinct modalities form disjoint angular regions on the hypersphere (see Table \ref{tab:dataset_similarity} in Supplementary Sec. \ref{sec:modality_clusters} for more details).}
    \label{fig:mind_the_gap}
    \vspace{-2em}
\end{figure}
\textbf{Geodesics for Matching:} \label{sec:mind_the_gap}  
Zero-shot transfer is hindered by the well-documented \emph{modality gap} \cite{levi2024double,liang2022mind}, where CLIP embeddings from text and image subspaces occupy disjoint regions of the hypersphere. Our analysis extends this: even within the image space, modality-specific clustering (e.g., ChestXRay, BrainMRI, Histopathology) confines embeddings to distinct angular regions $\mathbb{S}^{d-1}$. This creates large inter-modality angular gaps (Figure~\ref{fig:mind_the_gap}), making zero-shot anomaly detection particularly challenging. If parameters lie closer to one modality’s cluster, the semantic gap to others is amplified. Because CLIP embeddings are $\ell_2$-normalized, they reside on $\mathbb{S}^{d-1}$, where the intrinsic metric is the geodesic distance $d_G(f_1,f_2) \;=\; \arccos(f_1^\top f_2)$. Unlike Euclidean distance (underestimates large gaps) or cosine similarity (lacks metric structure), $d_G$ respects hyperspherical curvature, preserves neighborhoods, and faithfully represents modality separations. Under the bi-Lipschitzness assumption (Assumption~\ref{assump:metric}), this ensures stability of matching and underpins theoretical guarantees.  

\textbf{Centroid Initialization and Update:} \label{sec:geodesic_centroid}  
We initialize normal and anomaly centroids at $\mathbf{c}^{(0)}=\mathbf{0}$, avoiding warm-start bias. For a feature $f_j$ satisfying the geodesic matching criterion, centroids are updated by exponential moving average (EMA): $\tilde{\mathbf{c}}^{(t+1)} = \alpha \mathbf{c}^{(t)} + (1-\alpha) f_j, 
\quad 
\mathbf{c}^{(t+1)} = \frac{\tilde{\mathbf{c}}^{(t+1)}}{\|\tilde{\mathbf{c}}^{(t+1)}\|_2}, \quad \alpha=0.5$.
Projection ensures centroids remain on $\mathbb{S}^{d-1}$. Within the spherical cap where matching occurs, $d_G$ behaves like Euclidean distance up to small scaling, so all guarantees—positivity, ignorability, consistency, and finite-$K$ robustness (Theorem~\ref{thm:finite_robustness})—remain intact. The geometric advantage of $d_G$ is that centroid neighborhoods expand consistently with the hypersphere, allowing a tighter adaptive threshold $\tau$.  

\textbf{Geodesic Loss for Intra-/Inter-Cluster Structure:} \label{sec:geodesic_loss}  
We directly embed $d_G$ in the learning objective:  

\emph{Intra-cluster compactness:} $\mathcal{L}_{\text{intra}}
= \frac{1}{|\mathcal{S}^+|} \sum_{j \in \mathcal{S}^+} d_G(f_j,\mathbf{c}^+)^2
+ \frac{1}{|\mathcal{S}^-|} \sum_{j \in \mathcal{S}^-} d_G(f_j,\mathbf{c}^-)^2$.
\emph{Inter-cluster separation:} $\mathcal{L}_{\text{inter}} = - d_G(\mathbf{c}^+,\mathbf{c}^-)$.

Maximizing $\mathcal{L}_{\text{inter}}$ prevents mode collapse and guarantees identifiability. The final geodesic loss is $\mathcal{L}_{\text{geo}} = \lambda_1 \mathcal{L}_{\text{intra}} + \lambda_2 \mathcal{L}_{\text{inter}}, \lambda_1,\lambda_2>0$,
balancing compactness and separation. This directly addresses Lemma~\ref{lemma:multimodal_reduction}, which shows that minimal mode separation is necessary for consistent convergence.

\textbf{Variance-Aware Channel Attention (VACA):} \label{sec:variance_attention} 
Variance decomposition (Figure~\ref{fig:variance_plot}, Supplementary Sec.~\ref{sec:variance_supplementary}) reveals that only a subset of embedding dimensions carries class-specific signal \cite{yu2024clipceil}, while many encode modality-specific variance that induces hyperspherical clustering (Figure~\ref{fig:mind_the_gap}). We address this with \emph{variance-aware channel attention} (VACA), which explicitly upweights class-discriminative channels and suppresses modality-driven ones. Let patch embeddings be $P\in\mathbb{R}^{B\times N\times D}$ (batch $B$, patches $N$, channels $D$), normalized along the channel dimension, and let class embeddings for \emph{normal} and \emph{anomaly} be $T\in\mathbb{R}^{D\times 2}$ (L2-normalized). Broadcasting and taking their Hadamard product yields per-channel contributions, aggregated into a discriminative score vector $\Delta\in\mathbb{R}^D$. Channels with high $\Delta_d$ align with the anomaly/normal axis, while low-$\Delta_d$ reflect modality noise. A small MLP with Softplus output produces nonnegative gains, scaled by $\gamma$ to form channel weights $w=1+\gamma\,a$. Reweighted patches $\widetilde{P}=P\odot w$ emphasize class signal, while a variance surrogate $\mathsf{Var}_{\mathrm{disc}}$, computed from weighted $\Delta$, summarizes discriminative energy. This reweighting shrinks angular gaps between modalities and strengthens “local positivity’’ in hyperspherical neighborhoods, ensuring subsequent geodesic matching (Section~\ref{sec:geodesic_centroid}) operates on richer, domain-invariant features.

\textbf{Ablation experiments:} Figure~\ref{fig:ablation} and Supplementary Sec.~\ref{sec:ablation} highlight the effectiveness of geodesic distance in capturing faithful neighborhoods on the hypersphere, outperforming Euclidean and cosine metrics. When combined with VACA, DA scores improve consistently, as variance-aware reweighting suppresses modality-specific leakage and amplifies class-discriminative channels. This shifts modality clusters closer in the discriminative subspace, increasing successful cross-domain matching. Geometrically, this restores local domain overlap, directly supporting the \emph{positivity assumption} by ensuring that normal and anomaly features across modalities share sufficient support for non-zero inclusion probability. Given the scarcity of medical data, this guarantees that relevant samples remain available for matching.\footnote{Unlike a fixed threshold, $\tau$ is defined adaptively as the minimum distance of a sample to the centroids (See Supplementary Sec. \ref{sec:algorithm}).  This ensures that matching decisions evolve naturally as centroids are updated via exponential moving averages and as feature space transforms while training.} The overall framework is summarized in Supplementary Sec.~\ref{sec:algorithm} (Algorithm~\ref{alg:main_method}). 

\textbf{Empirical comparison with the existing methods:} Compared to the state of the art anomaly detection methods like MVFA \cite{huang2024adapting}, AnomalyCLIP \cite{zhou2023anomalyclip}, and continual learning frameworks like BiLORA \cite{zhu2025bilora} that utilize naive pooling for incremental learning setups where domains (finetuning as domains keeps on increasing), our proposed method performs better due to the inherent advantage of matching by filtering out the confounding data points (see Fig. \ref{fig:final}). We see that the DA score, which measures the dips in performance of the model as newer domains are introduced for finetuning (in accordance to Theorem \ref{thm:finite_robustness}), are always above 4 (a DA score of 4 or above indicates no decrease or increase in the model's performance as we keep on finetuning on newer domains) for our proposed method unlike the existing methods (detailed analysis is provided in Supplementary Sec. \ref{sec:sota}). Table \ref{tab:zero-shot} demonstartes that our method achieves comparable performance with respect to the state of the art zero-shot anomaly detection methods while also surpassing them in most cases. \textit{Takeaway:} The proposed Geodesic-aware Matching method is the only method with non-detoriative performance behaviour with addition of more data compared to exisitng baselines along with superior downstream performance.

\section{Conclusion}
Pooling heterogeneous datasets is attractive for representation learning, but we showed both theoretically and empirically that naive pooling or uniform subsampling can amplify distributional asymmetries and harm generalization, especially under domain addition. Our matching framework addresses this by adaptively selecting samples around hypersphere centroids, ensuring double robustness and filtering confounding domains that violate positivity. Practically, we instantiated these ideas in zero-shot medical anomaly detection, an extreme heterogeneity case where CLIP embeddings cluster by modality. Using geodesic distance as the intrinsic matching metric and augmenting it with VACA, we emphasized class-discriminative while suppressing modality-driven dimensions.

\textbf{Acknowledgments.} Prof. Lokhande acknowledges support from University at Buffalo startup funds, an Adobe Research Gift, an NVIDIA Academic Grant, and the National Center for Advancing Translational Sciences of the NIH (award UM1TR005296 to the University at Buffalo). Dr.\ Chakraborty performed this work under the auspices of the U.S.\ Department of Energy by Lawrence Livermore National Laboratory under Contract DE-AC52-07NA27344.

%%%%%%%%%%%%%%%%%%%%%%%%%%%%%%%%%%%%%%%%%%%%%%%%%%%%%%%%%%%%

\bibliographystyle{plain}
\bibliography{references}

@article{gulrajani2020search,
  title={In search of lost domain generalization},
  author={Gulrajani, Ishaan and Lopez-Paz, David},
  journal={arXiv preprint arXiv:2007.01434},
  year={2020}
}

@inproceedings{mahajan2021domain,
  title={Domain generalization using causal matching},
  author={Mahajan, Divyat and Tople, Shruti and Sharma, Amit},
  booktitle={International conference on machine learning},
  pages={7313--7324},
  year={2021},
  organization={PMLR}
}

@inproceedings{lokhande2022equivariance,
  title={Equivariance allows handling multiple nuisance variables when analyzing pooled neuroimaging datasets},
  author={Lokhande, Vishnu Suresh and Chakraborty, Rudrasis and Ravi, Sathya N and Singh, Vikas},
  booktitle={Proceedings of the IEEE/CVF Conference on Computer Vision and Pattern Recognition},
  pages={10432--10441},
  year={2022}
}

@article{mehta2023efficient,
  title={EFFICIENT DISCRETE MULTI-MARGINAL OPTIMAL TRANSPORT REGULARIZATION},
  author={Mehta, Ronak and Kline, Jeffery and Lokhande, Vishnu Suresh and Fung, Glenn and Singh, Vikas},
  year={2023},
  publisher={OpenReview. net}
}

@inproceedings{nazarovs2021graph,
  title={Graph reparameterizations for enabling 1000+ Monte Carlo iterations in Bayesian deep neural networks},
  author={Nazarovs, Jurijs and Mehta, Ronak R and Lokhande, Vishnu Suresh and Singh, Vikas},
  booktitle={Uncertainty in Artificial Intelligence},
  pages={118--128},
  year={2021},
  organization={PMLR}
}

@article{zhou2018statistical,
  title={Statistical tests and identifiability conditions for pooling and analyzing multisite datasets},
  author={Zhou, Hao Henry and Singh, Vikas and Johnson, Sterling C and Wahba, Grace and Alzheimer’s Disease Neuroimaging Initiative and Berkeley and Charles DeCArli},
  journal={Proceedings of the National Academy of Sciences},
  volume={115},
  number={7},
  pages={1481--1486},
  year={2018},
  publisher={National Academy of Sciences}
}

@inproceedings{huang2024adapting,
  title={Adapting visual-language models for generalizable anomaly detection in medical images},
  author={Huang, Chaoqin and Jiang, Aofan and Feng, Jinghao and Zhang, Ya and Wang, Xinchao and Wang, Yanfeng},
  booktitle={Proceedings of the IEEE/CVF Conference on Computer Vision and Pattern Recognition},
  pages={11375--11385},
  year={2024}
}

@article{shen2024data,
  title={The data addition dilemma},
  author={Shen, Judy Hanwen and Raji, Inioluwa Deborah and Chen, Irene Y},
  journal={arXiv preprint arXiv:2408.04154},
  year={2024}
}

@article{kim2016quasi,
  title={Quasi-experimental designs for causal inference},
  author={Kim, Yongnam and Steiner, Peter},
  journal={Educational psychologist},
  volume={51},
  number={3-4},
  pages={395--405},
  year={2016},
  publisher={Taylor \& Francis}
}

@article{thompson2014enigma,
  title={The ENIGMA Consortium: large-scale collaborative analyses of neuroimaging and genetic data},
  author={Thompson, Paul M and Stein, Jason L and Medland, Sarah E and Hibar, Derrek P and Vasquez, Alejandro Arias and Renteria, Miguel E and Toro, Roberto and Jahanshad, Neda and Schumann, Gunter and Franke, Barbara and others},
  journal={Brain imaging and behavior},
  volume={8},
  number={2},
  pages={153--182},
  year={2014},
  publisher={Springer}
}

@article{zhou2022domain,
  title={Domain generalization: A survey},
  author={Zhou, Kaiyang and Liu, Ziwei and Qiao, Yu and Xiang, Tao and Loy, Chen Change},
  journal={IEEE transactions on pattern analysis and machine intelligence},
  volume={45},
  number={4},
  pages={4396--4415},
  year={2022},
  publisher={IEEE}
}

@article{kingma2018glow,
  title={Glow: Generative flow with invertible 1x1 convolutions},
  author={Kingma, Durk P and Dhariwal, Prafulla},
  journal={Advances in neural information processing systems},
  volume={31},
  year={2018}
}

@article{dinh2016density,
  title={Density estimation using real nvp},
  author={Dinh, Laurent and Sohl-Dickstein, Jascha and Bengio, Samy},
  journal={arXiv preprint arXiv:1605.08803},
  year={2016}
}

@inproceedings{zhu2025bilora,
  title={BiLoRA: Almost-Orthogonal Parameter Spaces for Continual Learning},
  author={Zhu, Hao and Zhang, Yifei and Dong, Junhao and Koniusz, Piotr},
  booktitle={Proceedings of the Computer Vision and Pattern Recognition Conference},
  pages={25613--25622},
  year={2025}
}

@article{levi2024double,
  title={The double-ellipsoid geometry of clip},
  author={Levi, Meir Yossef and Gilboa, Guy},
  journal={arXiv preprint arXiv:2411.14517},
  year={2024}
}

@article{yu2024clipceil,
  title={Clipceil: Domain generalization through clip via channel refinement and image-text alignment},
  author={Yu, Xi and Yoo, Shinjae and Lin, Yuewei},
  journal={Advances in Neural Information Processing Systems},
  volume={37},
  pages={4267--4294},
  year={2024}
}

@article{rosenbaum1983central,
  title={The central role of the propensity score in observational studies for causal effects},
  author={Rosenbaum, Paul R and Rubin, Donald B},
  journal={Biometrika},
  volume={70},
  number={1},
  pages={41--55},
  year={1983},
  publisher={Oxford University Press}
}

@inproceedings{shalit2017estimating,
  title={Estimating Individual Treatment Effect: Generalization Bounds and Algorithms},
  author={Shalit, Uri and Johansson, Fredrik D. and Sontag, David},
  booktitle={Proceedings of the 34th International Conference on Machine Learning},
  pages={3076--3085},
  year={2017}
}

@article{dunipace2021optimal,
  title={Optimal Transport Weights for Causal Inference},
  author={Dunipace, Eric},
  journal={arXiv preprint arXiv:2109.01991},
  year={2021}
}

@article{gunsilius2021matching,
  title={Matching for Causal Effects via Multimarginal Unbalanced Optimal Transport},
  author={Gunsilius, Florian and Xu, Yuliang},
  journal={arXiv preprint arXiv:2112.04398},
  year={2021}
}

@article{liang2022mind,
  title={Mind the gap: Understanding the modality gap in multi-modal contrastive representation learning},
  author={Liang, Victor Weixin and Zhang, Yuhui and Kwon, Yongchan and Yeung, Serena and Zou, James Y},
  journal={Advances in Neural Information Processing Systems},
  volume={35},
  pages={17612--17625},
  year={2022}
}

@inproceedings{roth2022towards,
  title={Towards total recall in industrial anomaly detection},
  author={Roth, Karsten and Pemula, Latha and Zepeda, Joaquin and Sch{\"o}lkopf, Bernhard and Brox, Thomas and Gehler, Peter},
  booktitle={Proceedings of the IEEE/CVF conference on computer vision and pattern recognition},
  pages={14318--14328},
  year={2022}
}

@article{chen2023april,
  title={APRIL-GAN: A Zero-/Few-Shot Anomaly Classification and Segmentation Method for CVPR 2023 VAND Workshop Challenge Tracks 1\&2: 1st Place on Zero-shot AD and 4th Place on Few-shot AD},
  author={Chen, Xuhai and Han, Yue and Zhang, Jiangning},
  journal={arXiv preprint arXiv:2305.17382},
  year={2023}
}

@inproceedings{deng2022anomaly,
  title={Anomaly detection via reverse distillation from one-class embedding},
  author={Deng, Hanqiu and Li, Xingyu},
  booktitle={Proceedings of the IEEE/CVF conference on computer vision and pattern recognition},
  pages={9737--9746},
  year={2022}
}

@inproceedings{jeong2023winclip,
  title={Winclip: Zero-/few-shot anomaly classification and segmentation},
  author={Jeong, Jongheon and Zou, Yang and Kim, Taewan and Zhang, Dongqing and Ravichandran, Avinash and Dabeer, Onkar},
  booktitle={Proceedings of the IEEE/CVF Conference on Computer Vision and Pattern Recognition},
  pages={19606--19616},
  year={2023}
}

@inproceedings{salehi2021multiresolution,
  title={Multiresolution knowledge distillation for anomaly detection},
  author={Salehi, Mohammadreza and Sadjadi, Niousha and Baselizadeh, Soroosh and Rohban, Mohammad H and Rabiee, Hamid R},
  booktitle={Proceedings of the IEEE/CVF conference on computer vision and pattern recognition},
  pages={14902--14912},
  year={2021}
}

@inproceedings{bao2024bmad,
  title={Bmad: Benchmarks for medical anomaly detection},
  author={Bao, Jinan and Sun, Hanshi and Deng, Hanqiu and He, Yinsheng and Zhang, Zhaoxiang and Li, Xingyu},
  booktitle={Proceedings of the IEEE/CVF Conference on Computer Vision and Pattern Recognition},
  pages={4042--4053},
  year={2024}
}

@inproceedings{jiang2023multi,
  title={Multi-scale cross-restoration framework for electrocardiogram anomaly detection},
  author={Jiang, Aofan and Huang, Chaoqin and Cao, Qing and Wu, Shuang and Zeng, Zi and Chen, Kang and Zhang, Ya and Wang, Yanfeng},
  booktitle={International Conference on Medical Image Computing and Computer-Assisted Intervention},
  pages={87--97},
  year={2023},
  organization={Springer}
}

@inproceedings{bergmann2019mvtec,
  title={MVTec AD--A comprehensive real-world dataset for unsupervised anomaly detection},
  author={Bergmann, Paul and Fauser, Michael and Sattlegger, David and Steger, Carsten},
  booktitle={Proceedings of the IEEE/CVF conference on computer vision and pattern recognition},
  pages={9592--9600},
  year={2019}
}

@article{zhou2023anomalyclip,
  title={Anomalyclip: Object-agnostic prompt learning for zero-shot anomaly detection},
  author={Zhou, Qihang and Pang, Guansong and Tian, Yu and He, Shibo and Chen, Jiming},
  journal={arXiv preprint arXiv:2310.18961},
  year={2023}
}

@inproceedings{cao2024adaclip,
  title={Adaclip: Adapting clip with hybrid learnable prompts for zero-shot anomaly detection},
  author={Cao, Yunkang and Zhang, Jiangning and Frittoli, Luca and Cheng, Yuqi and Shen, Weiming and Boracchi, Giacomo},
  booktitle={European Conference on Computer Vision},
  pages={55--72},
  year={2024},
  organization={Springer}
}

@inproceedings{mao2025beyond,
  title={Beyond Single-Modal Boundary: Cross-Modal Anomaly Detection through Visual Prototype and Harmonization},
  author={Mao, Kai and Wei, Ping and Lian, Yiyang and Wang, Yangyang and Zheng, Nanning},
  booktitle={Proceedings of the Computer Vision and Pattern Recognition Conference},
  pages={9964--9973},
  year={2025}
}

@article{you2022unified,
  title={A unified model for multi-class anomaly detection},
  author={You, Zhiyuan and Cui, Lei and Shen, Yujun and Yang, Kai and Lu, Xin and Zheng, Yu and Le, Xinyi},
  journal={Advances in Neural Information Processing Systems},
  volume={35},
  pages={4571--4584},
  year={2022}
}

@article{li2025memory,
  title={Memory-statistics tradeoff in continual learning with structural regularization},
  author={Li, Haoran and Wu, Jingfeng and Braverman, Vladimir},
  journal={arXiv preprint arXiv:2504.04039},
  year={2025}
}

@inproceedings{motiian2017unified,
  title={Unified deep supervised domain adaptation and generalization},
  author={Motiian, Saeid and Piccirilli, Marco and Adjeroh, Donald A and Doretto, Gianfranco},
  booktitle={Proceedings of the IEEE international conference on computer vision},
  pages={5715--5725},
  year={2017}
}

@article{theisen2021evaluating,
  title={Evaluating state-of-the-art classification models against bayes optimality},
  author={Theisen, Ryan and Wang, Huan and Varshney, Lav R and Xiong, Caiming and Socher, Richard},
  journal={Advances in Neural Information Processing Systems},
  volume={34},
  pages={9367--9377},
  year={2021}
}

@inproceedings{huang2018neural,
  title={Neural autoregressive flows},
  author={Huang, Chin-Wei and Krueger, David and Lacoste, Alexandre and Courville, Aaron},
  booktitle={International conference on machine learning},
  pages={2078--2087},
  year={2018},
  organization={PMLR}
}

@article{lee2021universal,
  title={Universal approximation using well-conditioned normalizing flows},
  author={Lee, Holden and Pabbaraju, Chirag and Sevekari, Anish Prasad and Risteski, Andrej},
  journal={Advances in Neural Information Processing Systems},
  volume={34},
  pages={12700--12711},
  year={2021}
}

@book{diakonikolas2023algorithmic,
  title={Algorithmic high-dimensional robust statistics},
  author={Diakonikolas, Ilias and Kane, Daniel M},
  year={2023},
  publisher={Cambridge university press}
}

\clearpage
\appendix
\thispagestyle{empty}

% Supplementary material
\onecolumn
\aistatstitle{Supplementary Materials}
\section{Proofs of Theoretical Results}
\label{sec:proofs}

\subsection{Assumptions}
\label{sec:assumptions}
\begin{assumption}[Metric equivalence]\label{assump:metric}
Let \(M:\mathbb{R}^d\to[0,\infty)\) be the matching metric used by the algorithm. There exist constants \(0<m_M\le L_M<\infty\) such that $m_M\|v\|_2 \le M(v) \le L_M\|v\|_2 \qquad\text{for all }v\in\mathbb{R}^d$
\end{assumption}
Assumption~\ref{assump:metric}  states $M$ is locally equivalent to the Euclidean norm on the data manifold. This means that 
$M$ does not distort distances by more than a fixed multiplicative factor. If, for example, $M(v)$ were much larger or smaller than $\|v\|_2$, depending on $v$, matching could select domains with means very far from the centroid, breaking the tight concentration of the matched set. In practice, this means distances measured under $M$ cannot collapse to zero or blow up arbitrarily relative to $\ell_2$ distances: the constants $(m_M, L_M)$ control this distortion. Such equivalence guarantees that concentration arguments and robustness bounds established in Euclidean space remain valid when extended to $M$, since neighborhoods defined by $M$ correspond to neighborhoods on the manifold up to constant factors. Moreover, because compact manifolds admit finite bi-Lipschitz charts, this equivalence allows our theoretical results to generalize from Gaussian settings in $\mathbb{R}^d$ to arbitrary compact Riemannian manifolds, ensuring that matching remains stable and generalizable across domains.

\begin{assumption}[Moment and Tail Conditions]
\label{ass:moment_tail}
We assume all domain means $\mu_k$ are drawn i.i.d. from a meta-distribution $D$ with: a) Finite mean and covariance: $\mathbb{E}[\mu_k]=\mu_*$ and $\mathrm{Cov}(\mu_k)=\Sigma_{\boldsymbol{\mu}} < \infty$ b) Sub-exponential tails: $\exists C, \beta > 0$ such that $\mathbb{P}(\|\mu_k - \mu_*\| > t) \leq C \exp(-\beta t)$ c) Finite fourth moment: $\mathbb{E}[\|\mu_k - \mu_*\|_2^4] < \infty$ d) There exist a unit vector $v\in\mathbb{R}^d$, constants $p_0\in(0,1]$ and $\delta>0$, such that the half-space $S_{v,\delta} \;=\; \{\mu : (\mu-\mu_*)\cdot v\ge\delta\}$ satisfies $\mathbb{P}(\mu_k\in S_{v,\delta})\ge p_0$.
\end{assumption}
Throughout our proofs, we use (i) the Strong Law of Large Numbers (SLLN), (ii) Lévy's Continuity Theorem, which states that almost sure convergence of characteristic functions implies convergence in distribution. All applications assume the moment and tail conditions in Assumption~\ref{ass:moment_tail}.

\subsection{Proof of Theorem \ref{thm:asymptotic_behavior}}
\label{sec:proof_asymptotic_behavior}

\AsymptoticBehavior*

\begin{proof}
\textbf{Naive Pooling Convergence:} Let $\{\boldsymbol{\mu}_k\}_{k=1}^K$ be i.i.d. draws from $\mathscr{D}_{\boldsymbol{\mu}}$ with $\mathbb{E}[\boldsymbol{\mu}_k] = \boldsymbol{\mu}_*$ and $\mathrm{Cov}[\boldsymbol{\mu}_k] = \boldsymbol{\Sigma}_{\boldsymbol{\mu}}$. The characteristic function of $\hat P_{\mathrm{pool}}^{(K)}$ is given by $\phi_{P}^{(K)}(\mathbf{t}) = \mathbb{E}\left[\exp(i\mathbf{t}^\top\mathbf{x})\right] \quad \text{where} \quad \mathbf{x} \sim \hat P_{\mathrm{pool}}^{(K)}$

Since $\hat P_{\mathrm{pool}}^{(K)} = \frac{1}{\sum_{k=1}^K N_k} \sum_{k=1}^K \sum_{i=1}^{N_k} \delta(\mathbf{x} - \mathbf{x}_{k,i})$ and each $\mathbf{x}_{k,i} \sim Q_k = \mathcal{N}(\boldsymbol{\mu}_k, \sigma^2 \mathbf{I}_d)$, we have $\phi_{P}^{(K)}(\mathbf{t}) = \frac{1}{\sum_{k=1}^K N_k} \sum_{k=1}^K N_k \cdot \mathbb{E}_{\mathbf{x} \sim Q_k}\left[\exp(i\mathbf{t}^\top\mathbf{x})\right]
= \frac{1}{\sum_{k=1}^K N_k} \sum_{k=1}^K N_k \cdot \exp\left(i\mathbf{t}^\top\boldsymbol{\mu}_k - \tfrac{1}{2}\sigma^2\|\mathbf{t}\|^2\right)$

Assuming balanced domains where $N_k = N$ for all $k$ (or that $\frac{N_k}{\sum_{j=1}^K N_j} \to \frac{1}{K}$ as $K \to \infty$), we obtain $\phi_{P}^{(K)}(\mathbf{t}) = \frac{1}{K} \sum_{k=1}^K \exp\left(i\mathbf{t}^\top\boldsymbol{\mu}_k - \tfrac{1}{2}\sigma^2\|\mathbf{t}\|^2\right) + o(1)$

By the Strong Law of Large Numbers for i.i.d. random vectors: $\frac{1}{K} \sum_{k=1}^K \exp(i\mathbf{t}^\top \boldsymbol{\mu}_k) \xrightarrow{a.s.} \mathbb{E}_{\boldsymbol{\mu} \sim \mathscr{D}_{\boldsymbol{\mu}}}\left[ \exp(i\mathbf{t}^\top \boldsymbol{\mu}) \right] =: \phi_{\boldsymbol{\mu}}(\mathbf{t})$

Under the moment conditions ($\mathbb{E}[\boldsymbol{\mu}]=\boldsymbol{\mu}_*$, $\mathrm{Cov}[\boldsymbol{\mu}]=\boldsymbol{\Sigma}_{\boldsymbol{\mu}}$), the characteristic function admits the expansion: $\phi_{\boldsymbol{\mu}}(\mathbf{t}) = \exp\left(i\mathbf{t}^\top \boldsymbol{\mu}_* - \tfrac{1}{2}\mathbf{t}^\top \boldsymbol{\Sigma}_{\boldsymbol{\mu}} \mathbf{t} + o(\|\mathbf{t}\|^2)\right)$

Thus, for any fixed $\mathbf{t} \in \mathbb{R}^d$: $\phi_{P}^{(K)}(\mathbf{t}) \xrightarrow{a.s.} \exp\left(-\tfrac{1}{2}\sigma^2\|\mathbf{t}\|^2\right) \cdot \exp\left(i\mathbf{t}^\top \boldsymbol{\mu}_* - \tfrac{1}{2}\mathbf{t}^\top \boldsymbol{\Sigma}_{\boldsymbol{\mu}} \mathbf{t}\right) = \phi_{\infty}(\mathbf{t})$, where $\phi_{\infty}(\mathbf{t})$ is the characteristic function of $\mathcal{N}(\boldsymbol{\mu}_*, \sigma^2 \mathbf{I}_d + \boldsymbol{\Sigma}_{\boldsymbol{\mu}})$. By Lévy's Continuity Theorem, almost sure convergence of characteristic functions implies convergence in distribution $\hat P_{\mathrm{pool}}^{(K)} \xrightarrow{d} \mathcal{N}(\boldsymbol{\mu}_*, \sigma^2 \mathbf{I}_d + \boldsymbol{\Sigma}_{\boldsymbol{\mu}})$

\textbf{Uniform Subsampling Convergence:}
The uniform subsampling procedure selects domains uniformly at random and draws $n \ll N_k$ samples from each selected domain. The key insight is that this process is equivalent to $\hat P_{\mathrm{sub}}^{(n)}(\mathbf{x}) = \frac{1}{|I|\cdot n} \sum_{k \in I}\sum_{j=1}^n \delta(\mathbf{x} - \mathbf{x}_{k,j})$
where $I \subset \{1,\dots,K\}$ is a random subset of domains chosen uniformly.

This can be reinterpreted as: each subsampled point $\mathbf{x}_{k,j}$ is obtained by 1) Sampling a domain index $k \sim \mathrm{Uniform}\{1,\ldots,K\}$ (through the random selection of $I$) 2) Sampling $\mathbf{x} \sim Q_k$ within that domain.

As $K \to \infty$, the domain selection process converges to drawing $\boldsymbol{\mu} \sim \mathscr{D}_{\boldsymbol{\mu}}$ then $\mathbf{x} \sim \mathcal{N}(\boldsymbol{\mu}, \sigma^2 \mathbf{I}_d)$. The characteristic function of this limiting process is $\phi_{\mathrm{sub}}^{(\infty)}(\mathbf{t}) = \mathbb{E}_{\boldsymbol{\mu} \sim \mathscr{D}_{\boldsymbol{\mu}}}\left[ \mathbb{E}_{\mathbf{x} \sim \mathcal{N}(\boldsymbol{\mu}, \sigma^2 \mathbf{I}_d)}\left[\exp(i\mathbf{t}^\top\mathbf{x})\right] \right]
= \mathbb{E}_{\boldsymbol{\mu} \sim \mathscr{D}_{\boldsymbol{\mu}}}\left[ \exp(i\mathbf{t}^\top\boldsymbol{\mu} - \tfrac{1}{2}\sigma^2\|\mathbf{t}\|^2) \right]$. This matches $\phi_{\infty}(\mathbf{t})$ from the naive pooling case. The finite sample size $n$ affects the rate of convergence but not the asymptotic limit, as the empirical characteristic function converges to its expectation by the Law of Large Numbers.

\textbf{Matching Convergence:} We analyze the  matching distribution defined as $\hat P_{\mathrm{match}}^{(\tau)}(\mathbf{x}) = \frac{1}{\sum_{k \in \mathcal{S}^{(\tau)}} N_k} \sum_{k \in \mathcal{S}^{(\tau)}} \sum_{i=1}^{N_k} \delta(\mathbf{x} - \mathbf{x}_{k,i})$, where $\mathcal{S}^{(\tau)} = \{k : M(\boldsymbol{\mu}_k - \mathbf{c}_n) < \tau\}$ and $\mathbf{c}_n$ is the running centroid.

The characteristic function of the matched distribution is $\phi_{\mathrm{match}}^{(\tau)}(\mathbf{t}) = \mathbb{E}\left[\exp(i\mathbf{t}^\top\mathbf{x})\right] \quad \text{where} \quad \mathbf{x} \sim P_{\mathrm{match}}^{(\tau)}$

Since $\hat P_{\mathrm{match}}^{(\tau)}$ is an empirical mixture over selected domains, we have $\phi_{\mathrm{match}}^{(\tau)}(\mathbf{t}) = \frac{\sum_{k \in \mathcal{S}^{(\tau)}} N_k \cdot \mathbb{E}_{\mathbf{x} \sim \hat{Q}_k}[\exp(i\mathbf{t}^\top\mathbf{x})]}{\sum_{k \in \mathcal{S}^{(\tau)}} N_k}$
where $\hat{Q}_k$ is the empirical distribution of domain $k$. As $K \to \infty$ and for each domain $k$, as $N_k \to \infty$ (or under the assumption that domain sample sizes are sufficiently large), the empirical characteristic function converges: $\mathbb{E}_{\mathbf{x} \sim \hat{Q}_k}[\exp(i\mathbf{t}^\top\mathbf{x})] \to \mathbb{E}_{\mathbf{x} \sim Q_k}[\exp(i\mathbf{t}^\top\mathbf{x})] = \exp(i\mathbf{t}^\top\boldsymbol{\mu}_k - \tfrac{1}{2}\sigma^2\|\mathbf{t}\|^2)$

The matching criterion selects domains where $M(\boldsymbol{\mu}_k - \mathbf{c}_n) < \tau$. Given that $\mathbf{c}_n \xrightarrow{p} \boldsymbol{\mu}_*$ (which follows from the iterative centroid update converging to the target mean), we have $W_k = \mathbb{I}\{M(\boldsymbol{\mu}_k - \mathbf{c}_n) < \tau\} \xrightarrow{p} \mathbb{I}\{M(\boldsymbol{\mu}_k - \boldsymbol{\mu}_*) < \tau\}$

By the law of large numbers for the weighted average over selected domains: $\phi_{\mathrm{match}}^{(\tau)}(\mathbf{t}) \xrightarrow{p} \frac{\mathbb{E}_{\boldsymbol{\mu} \sim \mathscr{D}_{\boldsymbol{\mu}}}\left[ \mathbb{I}\{M(\boldsymbol{\mu} - \boldsymbol{\mu}_*) < \tau\} \cdot \exp(i\mathbf{t}^\top\boldsymbol{\mu} - \tfrac{1}{2}\sigma^2\|\mathbf{t}\|^2) \right]}{\mathbb{E}_{\boldsymbol{\mu} \sim \mathscr{D}_{\boldsymbol{\mu}}}\left[ \mathbb{I}\{M(\boldsymbol{\mu} - \boldsymbol{\mu}_*) < \tau\} \right]}$. As the matching process becomes increasingly selective (in the sense that $\tau$ can be chosen to decrease appropriately with $K$), we consider the limit $\lim_{\tau \to 0} \phi_{\mathrm{match}}^{(\tau)}(\mathbf{t}) = \frac{\mathbb{E}_{\boldsymbol{\mu} \sim \mathscr{D}_{\boldsymbol{\mu}}}\left[ \delta(\boldsymbol{\mu} - \boldsymbol{\mu}_*) \cdot \exp(i\mathbf{t}^\top\boldsymbol{\mu} - \tfrac{1}{2}\sigma^2\|\mathbf{t}\|^2) \right]}{\mathbb{E}_{\boldsymbol{\mu} \sim \mathscr{D}_{\boldsymbol{\mu}}}\left[ \delta(\boldsymbol{\mu} - \boldsymbol{\mu}_*) \right]}$, where $\delta(\cdot)$ represents the limiting density concentration at $\boldsymbol{\mu}_*$.

Under the condition that $\mathscr{D}_{\boldsymbol{\mu}}$ has positive density at $\boldsymbol{\mu}_*$ (which holds for continuous meta-distributions), we obtain $\phi_{\mathrm{match}}^{(\tau)}(\mathbf{t}) \to \exp(i\mathbf{t}^\top\boldsymbol{\mu}_* - \tfrac{1}{2}\sigma^2\|\mathbf{t}\|^2)$. This is the characteristic function of $\mathcal{N}(\boldsymbol{\mu}_*, \sigma^2\mathbf{I}_d)$. By Lévy's Continuity Theorem: $\hat P_{\mathrm{match}}^{(\tau)} \xrightarrow{d} \mathcal{N}(\boldsymbol{\mu}_*, \sigma^2\mathbf{I}_d)$. The double robustness property ensures that this convergence holds even with mild misspecification of the matching radius $\tau$, as long as the centroid estimation is consistent.
\end{proof}

\subsection{Proof of Corollary \ref{cor:exchangeability}}
\label{sec:cor_exchangeability_proof}
\begin{corollary}[\textbf{Exchangeability of Pooling Strategies}]
\label{cor:exchangeability}
Each pooling strategy is invariant to permutations of domain indices under the following conditions: a) \textbf{Naive Pooling:}  
    The pooled distribution $\hat P_{\mathrm{pool}}^{(K)}(\mathbf{x}) = \frac{1}{\sum_{k=1}^K N_k} \sum_{k=1}^K \sum_{i=1}^{N_k} \delta(\mathbf{x} - \mathbf{x}_{k,i})$ is permutation-invariant by commutativity of addition. Relabeling domain indices leaves the sum unchanged. b) \textbf{Uniform Subsampling:}  
    The subsampled distribution $\hat P_{\mathrm{sub}}^{(n)}(\mathbf{x}) = \frac{1}{|I|\cdot n} \sum_{k \in I}\sum_{j=1}^n \delta(\mathbf{x} - \mathbf{x}_{k,j})$,
    where $I \subset \{1,\dots,K\}$ is chosen uniformly at random, is exchangeable. Uniform sampling of domains preserves the joint law under permutations of indices. c) \textbf{Matching:}  
    The matched distribution $\hat P_{\mathrm{match}}^{(\tau)}(\mathbf{x}) = \frac{1}{\sum_{k \in \mathcal{S}^{(\tau)}} N_k} \sum_{k \in \mathcal{S}^{(\tau)}} \sum_{i=1}^{N_k} \delta(\mathbf{x} - \mathbf{x}_{k,i})$,
    where $\mathcal{S}^{(\tau)} = \{k : M(\boldsymbol{\mu}_k - \mathbf{c}_n) < \tau\}$, is exchangeable. The centroid $\mathbf{c}_n$ is permutation-invariant, and inclusion depends only on distances to this centroid, which are unaffected by relabeling.

Thus, all three pooling strategies are exchangeable, both in case of finite $K$ and as $K \to \infty$. \qed
\end{corollary}

\begin{proof} We fix any finite $K$ and any permutation $\pi:\{1,\dots,K\}\to\{1,\dots,K\}$. We show that permuting domain labels leaves each pooled empirical distribution unchanged in law; in fact, for (a) it is unchanged pointwise.

\paragraph{a) Naive pooling.}
Recall $\hat P_{\mathrm{pool}}^{(K)}(\mathbf{x})
= \frac{1}{\sum_{k=1}^K N_k}\sum_{k=1}^K\sum_{i=1}^{N_k}\delta(\mathbf{x}-\mathbf{x}_{k,i})$.
Under the relabeling $k\mapsto \pi(k)$ we obtain
\[
\frac{1}{\sum_{k=1}^K N_{\pi(k)}}\sum_{k=1}^K\sum_{i=1}^{N_{\pi(k)}}\delta(\mathbf{x}-\mathbf{x}_{\pi(k),i})
= \frac{1}{\sum_{k=1}^K N_k}\sum_{k=1}^K\sum_{i=1}^{N_k}\delta(\mathbf{x}-\mathbf{x}_{k,i})
= \hat P_{\mathrm{pool}}^{(K)}(\mathbf{x}),
\]
since $\{N_{\pi(k)}\}_{k=1}^K$ is just a reordering of $\{N_k\}_{k=1}^K$ and addition is commutative. Thus naive pooling is permutation-invariant (exchangeable) for every realization and hence in distribution.

\paragraph{(b) Uniform subsampling.}
Let $I\subset\{1,\dots,K\}$ be chosen uniformly at random and
\[
\hat P_{\mathrm{sub}}^{(n)}(\mathbf{x})
= \frac{1}{|I|\,n}\sum_{k\in I}\sum_{j=1}^n \delta(\mathbf{x}-\mathbf{x}_{k,j}).
\]
The law of $I$ is invariant under relabeling: for any permutation $\pi$, $\pi(I)$ is uniform iff $I$ is uniform. Therefore, conditional on the data $\{\mathbf{x}_{k,j}\}$, the distribution of $\frac{1}{|\pi(I)|\,n}\sum_{k\in \pi(I)}\sum_{j=1}^n \delta(\mathbf{x}-\mathbf{x}_{k,j})$
coincides with that of $\hat P_{\mathrm{sub}}^{(n)}(\mathbf{x})$. Hence $\hat P_{\mathrm{sub}}^{(n)}$ is exchangeable.

\paragraph{(c) Matching.}
Let $\mathcal{S}^{(\tau)}=\{\,k: M(\mu_k-\mathbf{c}_n)<\tau\,\},\qquad
\hat P_{\mathrm{match}}^{(\tau)}(\mathbf{x})
= \frac{1}{\sum_{k\in \mathcal{S}^{(\tau)}}N_k}\sum_{k\in \mathcal{S}^{(\tau)}}\sum_{i=1}^{N_k}\delta(\mathbf{x}-\mathbf{x}_{k,i})$.
The centroid $\mathbf{c}_n$ is a permutation-invariant statistic of the multiset of observed samples (and, recursively, of prior matched sets); in particular, it depends on domains only through symmetric sums/averages. Consequently, the inclusion set $\mathcal{S}^{(\tau)}$ depends on $\{\mu_k\}$ only via distances to $\mathbf{c}_n$ and is unchanged by relabeling domain indices. Therefore, $\hat P_{\mathrm{match}}^{(\tau)}$ is invariant under permutations and hence exchangeable.

In all three cases, permutation of domain labels leaves the empirical distribution unchanged (pointwise for naive pooling; in law for subsampling and matching). Thus each pooling strategy is exchangeable for any finite $K$, and with even stronger guarentees as $K\to\infty$.
\end{proof}

\subsection{Proof of Theorem \ref{thm:symmetric_convergence}}
\label{sec:thm_symmetric_convergence}

\SymmetricConvergence*

\begin{proof}
\textbf{Naive Pooling:} The naive pooled mean is defined as $\bar{\boldsymbol{\mu}}_K = \frac{1}{\sum_{k=1}^K N_k} \sum_{k=1}^K \sum_{i=1}^{N_k} \mathbf{x}_{k,i}$, where $\mathbf{x}_{k,i} \sim Q_k = \mathcal{N}(\boldsymbol{\mu}_k, \sigma^2\mathbf{I}_d)$. Taking expectation and using linearity: $\mathbb{E}[\bar{\boldsymbol{\mu}}_K] = \frac{1}{\sum_{k=1}^K N_k} \sum_{k=1}^K N_k \cdot \mathbb{E}[\mathbf{x}_{k,i}] = \frac{1}{\sum_{k=1}^K N_k} \sum_{k=1}^K N_k \cdot \mathbb{E}[\boldsymbol{\mu}_k]$. By Def.~\ref{def:sym_dif}, $\mathbb{E}[\boldsymbol{\mu}_k] = \boldsymbol{\mu}_*$ for all $k$, so $\mathbb{E}[\bar{\boldsymbol{\mu}}_K] = \frac{1}{\sum_{k=1}^K N_k} \sum_{k=1}^K N_k \cdot \boldsymbol{\mu}_* = \boldsymbol{\mu}_*$

\textbf{Uniform Subsampling:} The uniform subsampled mean is $\bar{\boldsymbol{\mu}}_I = \frac{1}{|I| \cdot n} \sum_{k \in I} \sum_{j=1}^n \mathbf{x}_{k,j}$, where $I \subset \{1,\dots,K\}$ is chosen uniformly at random. Taking expectation conditioned on the domain selection: $\mathbb{E}[\bar{\boldsymbol{\mu}}_I] = \mathbb{E}\left[ \frac{1}{|I| \cdot n} \sum_{k \in I} \sum_{j=1}^n \mathbf{x}_{k,j} \right] = \mathbb{E}\left[ \frac{1}{|I|} \sum_{k \in I} \boldsymbol{\mu}_k \right]$. Since domain selection is uniform and $\mathbb{E}[\boldsymbol{\mu}_k] = \boldsymbol{\mu}_*$ by symmetry $\mathbb{E}\left[ \frac{1}{|I|} \sum_{k \in I} \boldsymbol{\mu}_k \right] = \frac{1}{K} \sum_{k=1}^K \mathbb{E}[\boldsymbol{\mu}_k] = \frac{1}{K} \sum_{k=1}^K \boldsymbol{\mu}_* = \boldsymbol{\mu}_*$

\textbf{Matching:} With initialization $\mathbf{c}_n^{(0)} = \boldsymbol{\mu}_*$, the matching criterion $M(\boldsymbol{\mu}_k - \boldsymbol{\mu}_*) < \tau$ selects a symmetric set of domains around $\boldsymbol{\mu}_*$ due to Assumption~\ref{def:sym_dif}. The matched mean after one iteration is $\bar{\boldsymbol{\mu}}_{\mathcal{S}}^{(1)} = \frac{1}{\sum_{k \in \mathcal{S}^{(\tau)}} N_k} \sum_{k \in \mathcal{S}^{(\tau)}} \sum_{i=1}^{N_k} \mathbf{x}_{k,i}$. Taking expectation: $\mathbb{E}[\bar{\boldsymbol{\mu}}_{\mathcal{S}}^{(1)}] = \frac{1}{\sum_{k \in \mathcal{S}^{(\tau)}} N_k} \sum_{k \in \mathcal{S}^{(\tau)}} N_k \cdot \mathbb{E}[\boldsymbol{\mu}_k] = \boldsymbol{\mu}_*$
since the selected domain means are symmetrically distributed around $\boldsymbol{\mu}_*$. The updated centroid $\mathbf{c}_n^{(1)} = \bar{\boldsymbol{\mu}}_{\mathcal{S}}^{(1)}$ remains unbiased, and by induction, this property holds for all iterations. Thus, the matching process maintains $\mathbb{E}[\bar{\boldsymbol{\mu}}_{\mathcal{S}}] = \boldsymbol{\mu}_*$. For matching, convergence does not critically depend on initializing the centroid at the true mean $\mu_*$. 
Even with a misspecified initialization $\mathbf{c}_0 \neq \mu_*$, the matching rule admits correction: by choosing an appropriate threshold $\tau$, the procedure includes domains sufficiently close to $\mu_*$ so that their aggregate mean drives the centroid toward the target. 
Thus, the iterative updates ensure that $\sum_{k=1}^{K}\mu_k$ concentrates around $\mu_*$ as $K$ grows. 
In contrast, under real-world meta-distributions $\mathscr{D}_{\boldsymbol{\mu}}$, which are typically skewed, the $Q_k$ means may not be symmetrically distributed, and naive pooling or subsampling can remain biased away from $\mu_*$. 
This highlights matching’s robustness in aligning with the target distribution despite imperfect initialization.
\end{proof}

\subsection{Proof of Theorem \ref{thm:finite_robustness}}
\label{sec:thm_finite_robustness}

\FiniteRobustness*

\begin{proof}
We analyze the update from $K$ to $K{+}1$ domains. To simplify geometry, we always project onto the current error direction.

We define the current error vector $v_K := \bar\mu_K - \mu^\ast$ and its norm
$\varepsilon_K := \|v_K\|_2$. If $\varepsilon_K>0$, define the normalized direction
$u_K := v_K / \|v_K\|_2$; otherwise take any fixed unit vector. Note that $\langle v_K, u_K \rangle \;=\; \|v_K\|_2 \;=\; \varepsilon_K$, so the inner product with $u_K$ exactly recovers the error magnitude. For any new domain, we define its projection $\xi_{K+1} := \langle \mu_{K+1}-\mu^\ast,\,u_K\rangle$. This scalar $\xi_{K+1}$ captures how much the new domain pushes us along the error direction.

\textbf{(a) Naive Pooling.}
The pooled update is $\bar\mu_{K+1}^{\textsf{pool}}
= \frac{K}{K+1}\,\bar\mu_K + \frac{1}{K+1}\,\mu_{K+1}$.
Projecting onto $u_K$ (the current error direction) we get $\big\langle \bar\mu_{K+1}^{\textsf{pool}}-\mu^*,\,u_K\big\rangle
= \frac{K}{K+1}\langle \bar\mu_K-\mu^*,u_K\rangle
+ \frac{1}{K+1}\langle \mu_{K+1}-\mu^*,u_K\rangle$.
By definition, $\langle \bar\mu_K-\mu^*,u_K\rangle=\varepsilon_K$, so $\big\langle \bar\mu_{K+1}^{\textsf{pool}}-\mu^*,\,u_K\big\rangle
= \frac{K}{K+1}\,\varepsilon_K + \frac{1}{K+1}\,\xi_{K+1}$,
where $\xi_{K+1} := \langle \mu_{K+1}-\mu^*, u_K\rangle$.  
Since $\varepsilon_{K+1} \ge \langle \bar\mu_{K+1}^{\textsf{pool}}-\mu^*,\,u_K\rangle$, we obtain $\varepsilon_{K+1} \;\ge\; \frac{K}{K+1}\,\varepsilon_K + \frac{1}{K+1}\,\xi_{K+1}$.

Because $\mu_{K+1}\sim D_\mu$ and Assumption~\ref{ass:moment_tail} states $D_\mu$ has sub-exponential tails, the one-dimensional projection $\xi_{K+1}$ also has sub-exponential tails (linear functionals of a sub-exponential random vector are still sub-ex). Therefore for any $\delta\in(0,1)$, $|\xi_{K+1}| \le \beta \log(2/\delta)
\quad \text{with prob. at least } 1-\delta$.

So far, we treated $u_K$ as fixed, but in reality, $u_K$ depends on the random average $\bar\mu_K$. By concentration of $\bar\mu_K$ around $\mu^*$, we know $\|\bar\mu_K - \mu^\ast\|_2 = O\!\left(\sqrt{\tfrac{\log(1/\delta)}{K}}\right)$. This induces an additional $O(\sqrt{\tfrac{\log(1/\delta)}{K}})$ correction in the bound for $\varepsilon_{K+1}$, because the error direction $u_K$ may fluctuate slightly relative to $\mu^*$.

Thus we have the high-probability lower bound: $\varepsilon_{K+1}
\;\ge\; \frac{K}{K+1}\,\varepsilon_K - \frac{\beta\log(2/\delta)}{K+1}
- O\!\left(\sqrt{\tfrac{\log(1/\delta)}{K}}\right)$.

Assumption~\ref{ass:moment_tail} also includes the ``half-space clause,'' which says: there exists a unit vector $v$ and constants $\delta>0$, $p_0>0$ such that $\Pr\big((\mu-\mu^*)\cdot v \ge \delta\big) \ge p_0$.
In words, the meta-distribution $D_\mu$ always places nontrivial mass in some half-space away from $\mu^*$. This guarantees that with constant probability, the new domain $\mu_{K+1}$ lies in a harmful direction relative to $u_K$, producing $\xi_{K+1}\ge c>0$. Substituting back, $\varepsilon_{K+1}
\;\ge\; \varepsilon_K - \frac{\varepsilon_K-c}{K+1}$.
Thus, even for large $K$, one-step \emph{constant-size} increases ($\Theta(1)$ jumps) occur with non-vanishing probability. This establishes the possibility of unbounded per-step deterioration for naive pooling.

\medskip
\noindent\textbf{(b) Uniform subsampling.}
At step $K{+}1$, we select $m$ domains uniformly at random from the available ones, 
and from each chosen domain $j$ we draw $n$ i.i.d.\ samples $x_{j,1},\dots,x_{j,n}$.
We then compute per-domain sample means $\hat\mu_j \;:=\; \frac{1}{n}\sum_{i=1}^n x_{j,i},
\qquad
\mathbb E[\hat\mu_j \mid \mu_j] = \mu_j,
\qquad
\mathrm{Cov}(\hat\mu_j \mid \mu_j) = \tfrac{1}{n}\sigma^2 I_d$, 
and average them: $\hat\mu_{K+1}^{\textsf{sub}}
\;=\; \frac{1}{m}\sum_{j=1}^m \hat\mu_j$.

We want $\mathbb E\|\hat\mu_{K+1}^{\textsf{sub}}-\mu^*\|_2^2$.  
Add and subtract $\mu_j$ inside: $\hat\mu_j - \mu^*
= (\mu_j-\mu^*) + (\hat\mu_j-\mu_j)$. 
So, $\hat\mu_{K+1}^{\textsf{sub}}-\mu^*
= \frac{1}{m}\sum_{j=1}^m (\mu_j-\mu^*) + \frac{1}{m}\sum_{j=1}^m (\hat\mu_j-\mu_j)$. Squaring and taking expectation, using independence across $j$: $\mathbb E\|\hat\mu_{K+1}^{\textsf{sub}}-\mu^*\|_2^2
= \underbrace{\mathbb E\Big\|\frac{1}{m}\sum_{j=1}^m (\mu_j-\mu^*)\Big\|_2^2}_{\text{between-domain variance}}
+ \underbrace{\mathbb E\Big\|\frac{1}{m}\sum_{j=1}^m (\hat\mu_j-\mu_j)\Big\|_2^2}_{\text{within-domain noise}}$.

Because $\mu_j\sim D_\mu$ are i.i.d.\ with covariance $\Sigma_\mu$, $\mathbb E\Big\|\tfrac{1}{m}\sum_{j=1}^m (\mu_j-\mu^*)\Big\|_2^2
= \frac{1}{m}\,\mathrm{Tr}(\Sigma_\mu)$.
Since $\mathrm{Tr}(\Sigma_\mu) \le d\,\lambda_{\max}(\Sigma_\mu)$, we can bound $\frac{1}{m}\,\mathrm{Tr}(\Sigma_\mu) \;\le\; \frac{\lambda_{\max}(\Sigma_\mu)}{m}$ (Let $d=1$).

Each $(\hat\mu_j-\mu_j)$ has covariance $\frac{1}{n}\sigma^2 I_d$.  
Averaging $m$ of them gives variance $\frac{1}{m^2}\cdot m \cdot \frac{1}{n}\sigma^2 I_d$, so $\mathbb E\Big\|\tfrac{1}{m}\sum_{j=1}^m (\hat\mu_j-\mu_j)\Big\|_2^2
= \frac{1}{m}\cdot\frac{d\sigma^2}{n}$. Thus we get $\mathbb E\|\hat\mu_{K+1}^{\textsf{sub}}-\mu^*\|_2^2
\;\le\;
\frac{\lambda_{\max}(\Sigma_\mu)}{m} + \frac{d\sigma^2}{mn}$.

Taking square roots gives the bound on the expected error: $\mathbb E[\varepsilon_{K+1}]
\;\le\; \varepsilon_K + O\!\Big(\frac{1}{\sqrt{m}} \;\vee\; \frac{1}{\sqrt{mn}}\Big)$.
Suppose one of the $m$ sampled domains has mean $\mu_j$ far from $\mu^*$ with $\|\mu_j-\mu^*\|_2\ge c$.  
That domain contributes weight $1/m$, so in projection the error shifts by $\Omega(c/m)$.  
Thus per-step worst-case deterioration is $\Omega(1/m)$, independent of $n$ (since $n$ only reduces the within-domain noise term).

\textbf{(c) Matching.}
At step $K$, let $S_K$ be the set of matched domains and $\hat c_K = \frac{1}{|S_K|}\sum_{k\in S_K} \mu_k + \eta_K$ be the matched centroid, where $\eta_K$ is the estimation error due to finite within-domain samples.

By sub-exponential tails (Assumption~\ref{ass:moment_tail}), 
standard concentration bounds imply $\|\eta_K\|_2 \;\le\; C\sqrt{\frac{\log(1/\delta)}{|S_K|}}
\quad\text{with probability at least } 1-\delta$. If the new domain $\mu_{K+1}$ is rejected by the rule 
$I_{K+1}=0$ (because $M(\mu_{K+1}-\hat c_K)\ge\tau$),
then $S_{K+1}=S_K$ and $\hat c_{K+1}=\hat c_K$.  
Therefore, $\varepsilon_{K+1}
= \|\hat c_{K+1}-\mu^*\|_2
= \|\hat c_K-\mu^*\|_2
\le \varepsilon_K + C\sqrt{\tfrac{\log(1/\delta)}{|S_K|}}$. So the error increases only by the sampling noise. If the new domain is accepted ($I_{K+1}=1$), then the centroid updates to $\hat c_{K+1}
= \frac{|S_K|\,\hat c_K + \mu_{K+1}}{|S_K|+1}$.
Subtracting $\mu^*$: $\hat c_{K+1}-\mu^*
= (\hat c_K-\mu^*) + \frac{1}{|S_K|+1}(\mu_{K+1}-\hat c_K)$.

Taking norms and applying triangle inequality: $\|\hat c_{K+1}-\mu^*\|_2
\;\le\; \|\hat c_K-\mu^*\|_2
+ \frac{1}{|S_K|+1}\,\|\mu_{K+1}-\hat c_K\|_2$. 
By the inclusion condition $M(\mu_{K+1}-\hat c_K)<\tau$ and metric equivalence
$m_M\|v\|_2 \le M(v) \le L_M\|v\|_2$ ((Assumption~\ref{assump:metric})), we obtain $\|\mu_{K+1}-\hat c_K\|_2 \;\le\; \frac{\tau}{m_M}$. 
Hence, $\varepsilon_{K+1}
\;\le\; \varepsilon_K + \frac{\tau}{m_M(|S_K|+1)} 
+ C\sqrt{\tfrac{\log(1/\delta)}{|S_K|}}$.

Suppose the new domain is not too far from the truth: $M(\mu_{K+1}-\mu^*) \le \varepsilon_K + \tfrac{\tau}{L_M}$. Then the perturbation term $\frac{\tau}{m_M(|S_K|+1)}$ shrinks like $1/|S_K|$, while the noise term shrinks like $1/\sqrt{|S_K|}$. Both vanish as $|S_K|$ grows.  
Therefore, with probability at least $1-\exp(-\Omega(|S_K|))$,
the new domain reduces or preserves the error: $\Pr(\varepsilon_{K+1}\le \varepsilon_K) \;\ge\; 1-\exp(-\Omega(|S_K|))$.
\end{proof}

Table \ref{tab:theoretical_comparison} summarizes the key points from Theorem \ref{thm:finite_robustness}.
\begin{table*}[t]
\centering
\caption{Comparison of error dynamics under domain addition. Only matching provides bounded deterioration, robustness to outliers, and explicit non-deterioration guarantees.}
\label{tab:theoretical_comparison}
\scriptsize
\begin{tabular}{lccc}
\toprule
\textbf{Strategy} & \textbf{Expected Error Change} & \textbf{Worst-case Deterioration} & \textbf{No detoriation} \\
\midrule
Naive pooling &
$\varepsilon_K + O(1)$ &
$\Theta(1)$ (constant jumps, unbounded in $K$) &
No guarantee \\
\addlinespace
Uniform subsampling &
$\varepsilon_K + O\!\left(\tfrac{1}{\sqrt{m}} \vee \tfrac{1}{\sqrt{mn}}\right)$ &
$\Omega(1/m)$ per step &
No guarantee \\
\addlinespace
Causal matching &
$\varepsilon_K + O\!\left(\tfrac{1}{\sqrt{|S_K|}}\right)$ &
$O(1/|S_K|)$ (bounded, vanishes with set size) &
Yes: non-deterioration w.h.p. \footnote{w.h.p. = with high probability, i.e. with probability at least 
$\Pr(\varepsilon_{K+1}\le \varepsilon_K) \;\ge\; 1-\exp(-\Omega(|S_K|))$} \\
\bottomrule
\end{tabular}
\end{table*}

\subsection{Proof of Corollary \ref{cor:optimal_tau}}
\label{sec:proof_optimal_tau}
\begin{corollary}[\textbf{Optimal $\tau$}]
\label{cor:optimal_tau}
Let $\hat c_K$ be the matched centroid after $K$ domains and $\epsilon_K=\|\hat c_K-\mu^*\|_2$. Under Assumption~\ref{assump:metric} (bi-Lipschitz metric $M$ with constants $m_M,L_M$), the matching rule
$\;I_{K+1}=\mathbf{1}\{M(\mu_{K+1}-\hat c_K)<\tau\}\;$ has the following properties: a) \textbf{Safe exclusion (sufficiency).} If 
$\|\mu_{K+1}-\mu^*\|_2 \;>\; \epsilon_K + \tau/m_M,$
then $M(\mu_{K+1}-\hat c_K)>\tau$, so the domain is \emph{rejected}. b) \textbf{Guaranteed inclusion (sufficiency).} If 
$\|\mu_{K+1}-\mu^*\|_2 \;\le\; \tau/L_M - \epsilon_K,$
then $M(\mu_{K+1}-\hat c_K)\le \tau$, so the domain is \emph{accepted}.$\!$
This inclusion band is nonempty iff $\;\tau> L_M \epsilon_K$.

In particular, any choice of $\tau> L_M \epsilon_K$ simultaneously (i) excludes all domains farther than $\epsilon_K+\tau/m_M$ from $\mu^*$ and (ii) guarantees inclusion of all domains within $\tau/L_M-\epsilon_K$ of $\mu^*$. Hence, choosing $\tau$ slightly larger than $L_M \epsilon_K$ yields a non-deterioration regime while still admitting sufficiently close (beneficial) domains.
\end{corollary}

\begin{proof}
Let $\hat c_K$ denote the matched centroid after $K$ domains, and let $\epsilon_K \;:=\; \|\hat c_K - \mu^*\|_2$. Assumption~\ref{assump:metric} (bi-Lipschitz equivalence of $M$ to the Euclidean norm) states that there exist constants $0<m_M\le L_M < \infty$ such that, for all $v\in\mathbb{R}^d$, $m_M \,\|v\|_2 \;\le\; M(v) \;\le\; L_M \,\|v\|_2$.
Recall that the matching rule includes a new domain $K{+}1$ iff $I_{K+1} \;=\; \mathbf{1}\!\left\{\, M(\mu_{K+1}-\hat c_K) \;<\; \tau \,\right\} \;=\; 1$. We prove the two sufficient conditions.

\textbf{a) Safe exclusion.}
Assume the new domain is \emph{far} from the target in Euclidean distance:
$\|\mu_{K+1} - \mu^*\|_2 \;>\; \epsilon_K + \frac{\tau}{m_M}$ (harmfulness assumption).
We will show this implies $M(\mu_{K+1}-\hat c_K)>\tau$, i.e., the domain is \emph{rejected}.

By the reverse triangle inequality (in normed spaces), $\|\mu_{K+1} - \hat c_K\|_2 
\;\ge\; \|\mu_{K+1} - \mu^*\|_2 \;-\; \|\hat c_K - \mu^*\|_2
\;=\; \|\mu_{K+1} - \mu^*\|_2 \;-\; \epsilon_K$.
Using the harmfulness assumption, $\|\mu_{K+1} - \hat c_K\|_2 
\;>\; \Big(\epsilon_K + \frac{\tau}{m_M}\Big) \;-\; \epsilon_K
\;=\; \frac{\tau}{m_M}$.
Now applying the \emph{lower} Lipschitz bound with $v=\mu_{K+1}-\hat c_K$: $M(\mu_{K+1} - \hat c_K)
\;\ge\; m_M\,\|\mu_{K+1} - \hat c_K\|_2
\;>\; m_M \cdot \frac{\tau}{m_M}
\;=\; \tau$.
Therefore $M(\mu_{K+1}-\hat c_K)>\tau$, so the matching rule rejects this domain ($I_{K+1}=0$).

\textbf{b) Guaranteed inclusion.}
Assume the new domain is \emph{close} to the target in Euclidean distance: 
$\|\mu_{K+1} - \mu^*\|_2 \;\le\; \frac{\tau}{L_M} \;-\; \epsilon_K$ (beneficial assumption). Note that the right-hand side is nonnegative iff $\tau > L_M \epsilon_K$; this is exactly the condition under which the ``beneficial band'' is nonempty. We will show that the beneficial condition implies $M(\mu_{K+1}-\hat c_K)\le \tau$, i.e., the domain is \emph{accepted}.

By the (forward) triangle inequality, $\|\mu_{K+1} - \hat c_K\|_2 
\;\le\; \|\mu_{K+1} - \mu^*\|_2 \;+\; \|\hat c_K - \mu^*\|_2
\;=\; \|\mu_{K+1} - \mu^*\|_2 \;+\; \epsilon_K$.
Using the beneficialness assumption, $\|\mu_{K+1} - \hat c_K\|_2 
\;\le\; \Big(\frac{\tau}{L_M} - \epsilon_K\Big) \;+\; \epsilon_K
\;=\; \frac{\tau}{L_M}$.
Now apply the \emph{upper} Lipschitz bound in Assumption~\ref{assump:metric} with $v=\mu_{K+1}-\hat c_K$: $M(\mu_{K+1} - \hat c_K)
\;\le\; L_M\,\|\mu_{K+1} - \hat c_K\|_2
\;\le\; L_M \cdot \frac{\tau}{L_M}
\;=\; \tau$. Therefore $M(\mu_{K+1}-\hat c_K)\le\tau$, so the matching rule accepts this domain ($I_{K+1}=1$). This proves the guaranteed-inclusion part. The inclusion band is nonempty precisely when $\tau > L_M \epsilon_K$.
\end{proof}

\subsection{Proof of Theorem \ref{thm:flow_bridge}}
\label{sec:thm_flow_bridge}

\FlowBridge*

\begin{proof}
By the universal approximation theorem for normalizing flows \cite{huang2018neural, lee2021universal}, 
for any continuous $p_{\text{data}}$ with compact support and any $\epsilon>0$, 
there exists a normalizing flow $T$ (e.g., a composition of coupling layers \cite{dinh2016density}) such that 
$D_{\text{KL}}(T_\# p_Z \,\|\, p_{\text{data}}) < \epsilon$.
This establishes the universal approximation property, showing that normalizing flows can transform the simple Gaussian distribution $p_Z$ to arbitrarily approximate the complex data distribution $p_{\text{data}}$.

The Lipschitz constant of a normalizing flow composed of $L$ coupling layers satisfies $L_T \leq \prod_{l=1}^L \|J_{T_l}\|_{\text{op}}$, where $\|J_{T_l}\|_{\text{op}}$ is the operator norm of the Jacobian of the $l$-th layer. For commonly used flow architectures like RealNVP \cite{dinh2016density} and GLOW \cite{kingma2018glow}, each coupling layer has a triangular Jacobian with bounded diagonal entries. Specifically, for affine coupling layers, the Jacobian determinant is bounded by $\exp(\|s\|_\infty)$, where $s$ is the scale network. By constraining the weights of the scale network (e.g., through spectral normalization or weight clipping), we can ensure $\|J_{T_l}\|_{\text{op}} \leq C_l$ for each layer, where $C_l$ represents the maximum operator norm of the Jacobian for the l-th layer. The overall Lipschitz constant then satisfies $L_T \leq \prod_{l=1}^L C_l$. The diameter of the data support appears because the flow needs to map the unit Gaussian to a region of size $\text{diam}(\text{supp}(p_{\text{data}}))$, which imposes a lower bound on the required expansion. The constant $C_{\text{flow}}$ captures the architectural constraints and can be made explicit for specific flow constructions.

Let $\mathcal{S}_Z = \{z \in \mathbb{R}^d : \|z - \mathbf{c}_z\|_2 < \tau\}$ be a matched set in the latent space with centroid $\mathbf{c}_z$. The transformed set is $\mathcal{S}_X = T(\mathcal{S}_Z) = \{T(z) : z \in \mathcal{S}_Z\}$. Since $T$ is $L_T$-Lipschitz continuous, for any $z_1, z_2 \in \mathbb{R}^d$, we have $\|T(z_1) - T(z_2)\|_2 \leq L_T \|z_1 - z_2\|_2$. Now, for any $z \in \mathcal{S}_Z$, we have $\|z - \mathbf{c}_z\|_2 < \tau$. Then, by the Lipschitz continuity of $T$, $\|T(z) - T(\mathbf{c}_z)\|_2 \leq L_T \|z - \mathbf{c}_z\|_2 < L_T \tau$. Therefore, for every $x \in \mathcal{S}_X$, we have $\|x - T(\mathbf{c}_z)\|_2 < L_T \tau$. It follows that $\frac{1}{|\mathcal{S}_X|} \sum_{x \in \mathcal{S}_X} \|x - T(\mathbf{c}_z)\|_2^2 < \frac{1}{|\mathcal{S}_X|} \sum_{x \in \mathcal{S}_X} (L_T \tau)^2 = L_T^2 \tau^2$. This establishes the variance preservation property, showing that the geometric structure of matched sets is preserved under the flow transformation.

Let $\bar{x} = \frac{1}{|\mathcal{S}_X|} \sum_{x \in \mathcal{S}_X} x$ be the empirical mean of the transformed set. Note that $\bar{x} = \frac{1}{|\mathcal{S}_Z|} \sum_{z \in \mathcal{S}_Z} T(z)$ because $T$ is a bijection and $\mathcal{S}_X = T(\mathcal{S}_Z)$. We want to bound $\|\bar{x} - T(\mathbf{c}_z)\|_2$. Using the linearity of the mean, we have $\bar{x} - T(\mathbf{c}_z) = \frac{1}{|\mathcal{S}_Z|} \sum_{z \in \mathcal{S}_Z} \left( T(z) - T(\mathbf{c}_z) \right)$. Taking the norm and using the triangle inequality, we get $\|\bar{x} - T(\mathbf{c}_z)\|_2 \leq \frac{1}{|\mathcal{S}_Z|} \sum_{z \in \mathcal{S}_Z} \| T(z) - T(\mathbf{c}_z) \|_2$. Now, for each $z \in \mathcal{S}_Z$, by the Lipschitz continuity of $T$, we have $\| T(z) - T(\mathbf{c}_z) \|_2 \le L_T \| z - \mathbf{c}_z \|_2 \le L_T \tau \quad\Rightarrow\quad
\frac{1}{|\mathcal{S}_X|} \sum_{x \in \mathcal{S}_X} \|x - T(\mathbf{c}_z)\|_2^2 \le L_T^2 \tau^2.$ Therefore, $\|\bar{x} - T(\mathbf{c}_z)\|_2 < \frac{1}{|\mathcal{S}_Z|} \sum_{z \in \mathcal{S}_Z} L_T \tau = L_T \tau$.  This establishes the concentration guarantee, showing that the empirical mean of the transformed set remains close to the transformed centroid, with the deviation bounded by the product of the Lipschitz constant and the matching radius.

The theorem demonstrates that the favorable properties of matching—variance reduction and concentration—extend from Gaussian to real-world non-Gaussian data through normalizing flows, with explicit control over the distortion via the Lipschitz constant $L_T$.
\end{proof}

\subsection{Proof of Lemma \ref{lemma:multimodal_reduction}}
\label{sec:proof_multimodal_reduction}
\begin{lemma}[\textbf{Multimodal Reduction to Unimodal Problems}]
\label{lemma:multimodal_reduction}
Assume the modes are separated as 
$\min_{i\neq j}\|\mu_i^*-\mu_j^*\|_2 > 2(\tau_{\max}+R_{\max}+\epsilon_{\max})$,
where $\tau_{\max}=\max_m \tau_m$, $R_{\max}$ is a (deterministic or high-probability) within-mode radius, and 
$\epsilon_{\max}=\max_m \epsilon_{K,m}$.
Assume coverage: $\tau_m \ge R_{\max}+\epsilon_{K,m}$ for each $m$.
Then a) Samples from different modes are never matched to the same centroid. b) The multimodal matching decomposes into $M$ independent unimodal problems. c) For each mode $m$, convergence holds as in Theorem~\ref{thm:asymptotic_behavior}. d) Non-deterioration per mode: $\epsilon_{K+1,m} \le \epsilon_{K,m}$ for all $m$.
\end{lemma}

\begin{proof}
\textbf{Disjoint Matching Sets:}
Let $\mathbf{x}$ be a sample from mode $i$ with true mean $\boldsymbol{\mu}_i^*$. By definition, $\|\mathbf{x} - \boldsymbol{\mu}_i^*\|_2 \leq R_{\max}$. The distance from $\mathbf{x}$ to its assigned centroid $\mathbf{c}_{n,i}$ satisfies $\|\mathbf{x} - \mathbf{c}_{n,i}\|_2 \leq \|\mathbf{x} - \boldsymbol{\mu}_i^*\|_2 + \|\boldsymbol{\mu}_i^* - \mathbf{c}_{n,i}\|_2 \leq R_{\max} + \epsilon_{K,i}$. By coverage, $R_{\max}+\epsilon_{K,i}\le \tau_i$, hence $\mathbf{x}$ is eligible for mode $i$. Now consider the distance from $\mathbf{x}$ to any other mode's centroid $\mathbf{c}_{n,j}$ for $j \neq i$:
\begin{align*}
\|\mathbf{x} - \mathbf{c}_{n,j}\|_2 &\geq \|\boldsymbol{\mu}_i^* - \boldsymbol{\mu}_j^*\|_2 - \|\mathbf{x} - \boldsymbol{\mu}_i^*\|_2 - \|\boldsymbol{\mu}_j^* - \mathbf{c}_{n,j}\|_2 \\
&> 2(\tau_{\max} + R_{\max} + \epsilon_{\max}) - R_{\max} - \epsilon_{\max} \\
&= \tau_{\max} + R_{\max} + \epsilon_{\max}.
\end{align*}

Meanwhile, the matching criterion for mode $i$ requires $\|\mathbf{x} - \mathbf{c}_{n,i}\|_2 \leq \tau_i \leq \tau_{\max}$. Therefore $\|\mathbf{x} - \mathbf{c}_{n,j}\|_2 > \tau_{\max} + R_{\max} + \epsilon_{\max} \geq \tau_{\max} \geq \tau_i$,
which means $\mathbf{x}$ cannot be matched to centroid $\mathbf{c}_{n,j}$. Thus, samples from different modes are never matched to the same centroid.

\textbf{Independence of Modal Problems:}
Since the matching sets for different modes are disjoint, the centroid updates for each mode depend only on samples from that mode: $\mathbf{c}_{n,m}^{(t+1)} = \frac{1}{|\mathcal{S}_m^{(t)}|} \sum_{\mathbf{x} \in \mathcal{S}_m^{(t)}} \mathbf{x}$, where $\mathcal{S}_m^{(t)}$ contains only samples from mode $m$. This means the $M$ matching processes evolve independently, with no cross-talk between modes.

\textbf{Unimodal Convergence for Each Mode:} For each mode $m$, we have an independent unimodal matching problem with: a) Target distribution: $\Norm(\boldsymbol{\mu}_m^*, \sigma_m^2 \I_d)$
b) Domain distributions: $\{Q_{k,m}\}_{k=1}^K$ where $Q_{k,m}$ is the restriction of domain $k$ to mode $m$
c) Matching radius: $\tau_m$

By Theorem~\ref{thm:asymptotic_behavior}, as $K \to \infty$, the matching process for mode $m$ converges: $P_{\mathrm{match},m}^{(\tau_m)} \dconv \Norm(\boldsymbol{\mu}_m^*, \sigma_m^2 \I_d)$. Since the modes are independent and the matching processes are disjoint, the overall multimodal distribution converges to the target: $\sum_{m=1}^M \pi_m P_{\mathrm{match},m}^{(\tau_m)} \dconv \sum_{m=1}^M \pi_m \Norm(\boldsymbol{\mu}_m^*, \sigma_m^2 \I_d) = \mathscr{D}_{\mathrm{test}}$.

\textbf{Non-Deterioration Guarantee per Mode:}
Since the modes behave in accordance with unimodal convergence when they are appropriately separated (as seen above), each mode’s matching problem reduces to an independent unimodal case. As we have already established, in the unimodal case (Corollary~\ref{cor:optimal_tau}), selection of $\tau$ to an optimal value guarantees non-deterioration
\end{proof}

\section{Properties of the Matched Distribution}
\label{sec:matched_properties}

The matched set $\mathcal{S} = \{ \mathbf{x} : M(\mathbf{x} - \mathbf{c}_n) < \tau \}$ exhibits these key properties that support the iterative refinement process and theoretical guarantees established in Theorems~\ref{thm:asymptotic_behavior} and~\ref{thm:finite_robustness}:

\begin{enumerate}
    \item \textbf{Monotonicity and Set Size Bounds:}
    Under Assumption~\ref{assump:metric}, the size of the matched set $|\mathcal{S}|$ is a monotonic function of the threshold $\tau$. Formally, for $0 \leq \tau_1 < \tau_2$:
    \[
    \mathcal{S}(\tau_1) \subseteq \mathcal{S}(\tau_2) \quad \text{and} \quad |\mathcal{S}(\tau_1)| \leq |\mathcal{S}(\tau_2)|.
    \]
    \textit{Proof:} For any sample $\mathbf{x}$ satisfying $M(\mathbf{x} - \mathbf{c}_n) < \tau_1$, it necessarily satisfies $M(\mathbf{x} - \mathbf{c}_n) < \tau_2$ since $\tau_1 < \tau_2$. The metric equivalence (Assumption~\ref{assump:metric}) ensures this ordering is preserved across the entire data space.
    
    \item \textbf{Concentration of Norms:}
    Under Assumption~\ref{assump:metric}, the average norm of samples in $\mathcal{S}$ is bounded by:
    \[
    \left| \frac{1}{|\mathcal{S}|} \sum_{\mathbf{x} \in \mathcal{S}} \|\mathbf{x}\|_2 - \|\mathbf{c}_n\|_2 \right| \leq L_M \cdot \tau.
    \]
    \textit{Proof:} By the reverse triangle inequality and Lipschitz continuity from Assumption~\ref{assump:metric}, for any $\mathbf{x} \in \mathcal{S}$:
    \[
    \left| \|\mathbf{x}\|_2 - \|\mathbf{c}_n\|_2 \right| \leq \|\mathbf{x} - \mathbf{c}_n\|_2 \leq L_M \cdot M(\mathbf{x} - \mathbf{c}_n) < L_M \cdot \tau.
    \]
    
    \item \textbf{Geometric Decay of Set Size:}
    Suppose the sample distribution has a continuous density $f(x)$ at $\mathbf{c}_n$, 
    and Assumption~\ref{assump:metric} holds. 
    Then for small $\tau$, the probability mass of the matched set satisfies
    \[
    \mathbb{P}\!\left( M(X - \mathbf{c}_n) < \tau \right) \asymp \tau^d,
    \]
    where $d$ is the ambient dimension. 
    Consequently, the expected size of the matched set among $N$ i.i.d. samples scales as
    \[
    \mathbb{E}[|\mathcal{S}|] \asymp N \tau^d.
    \]
    
    \textit{Justification:} The metric equivalence guarantees that the $M$-ball 
    $\{x : M(x - \mathbf{c}_n) < \tau\}$ lies between Euclidean balls of radii 
    $\tau / L_M$ and $\tau / m_M$. 
    The volume of such balls scales as $\Theta(\tau^d)$, 
    and continuity of $f(x)$ at $\mathbf{c}_n$ ensures the probability mass inside scales the same. 
    Thus, the expected count in $\mathcal{S}$ is $N$ times this probability.
    
    \item \textbf{Variance Reduction:}
    Under Assumption~\ref{assump:metric}, the empirical variance of the data in $\mathcal{S}$ is bounded by:
    \[
    \frac{1}{|\mathcal{S}|} \sum_{\mathbf{x} \in \mathcal{S}} \|\mathbf{x} - \bar{\boldsymbol{\mu}}_{\mathcal{S}}\|_2^2 \leq (L_M \cdot \tau)^2.
    \]
    \textit{Proof:} Since all samples in $\mathcal{S}$ satisfy $M(\mathbf{x} - \mathbf{c}_n) < \tau$, and by Lipschitz continuity $\|\mathbf{x} - \mathbf{c}_n\|_2 \leq L_M \cdot \tau$, the maximum possible variance around any point in $\mathcal{S}$ is $(L_M \cdot \tau)^2$.
\end{enumerate}

These properties provide the mathematical foundation for the robustness guarantees in Theorem~\ref{thm:finite_robustness}:
\begin{itemize}
    \item \textbf{Monotonicity} ensures stable iterative refinement under Assumption~\ref{assump:metric}
    \item \textbf{Concentration} guarantees coherent matched sets using the Lipschitz property
    \item \textbf{Geometric decay} explains sample efficiency trade-offs under the tail conditions
    \item \textbf{Variance reduction} directly enables bounded error guarantees through metric equivalence
\end{itemize}

\section{Empirical Analysis}
\label{sec:ablation_exp}

\subsection{Synthetic Validation of Theoretical Guarantees}
\label{sec:synthetic}

We conduct synthetic experiments to validate Theorems~\ref{thm:asymptotic_behavior}, \ref{thm:symmetric_convergence}, and \ref{thm:finite_robustness}. Throughout, the target test distribution is $\mathscr{D}_{\mathrm{test}}=\mathcal{N}(\boldsymbol{\mu}_*,\sigma^2 \mathbf{I}_d)$ with $\boldsymbol{\mu}_*=\mathbf{0}$ and $\sigma=0.8$. Each training domain is $Q_k=\mathcal{N}(\boldsymbol{\mu}_k,\sigma^2\mathbf{I}_d)$, where domain means $\{\boldsymbol{\mu}_k\}$ are drawn from a meta-distribution $\mathscr{D}_{\boldsymbol{\mu}}$ that satisfies Assumption~\ref{ass:moment_tail} (finite moments, light tails) and—when we model asymmetry—the directional-mass condition (Sec.~\ref{sec:theory}). Unless stated otherwise, asymmetric settings use an asymmetry strength $\alpha=1.5$. All results are averaged over 10 random seeds.

\begin{figure}[t]
    \centering
    \includegraphics[width=0.7\textwidth]{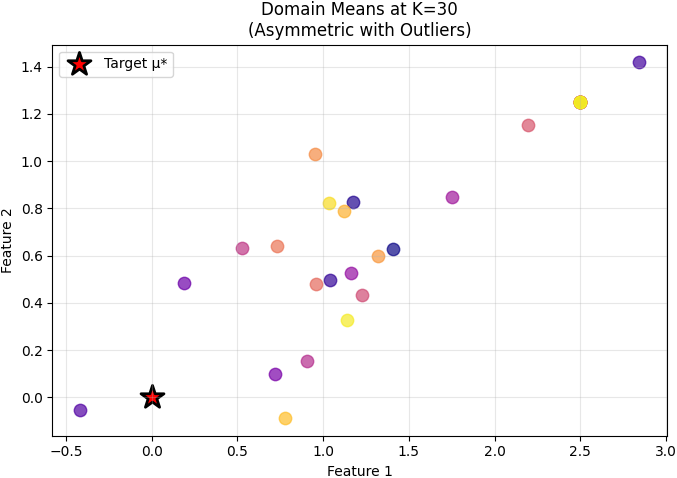}
    \caption{Asymmetric setting: domain means (colored dots) concentrate along a biased direction, away from the target mean $\boldsymbol{\mu}_*=\mathbf{0}$ (red star).}
    \label{fig:synthetic_intro}
\end{figure}

\paragraph{Implementation details.}
We instantiate the three strategies exactly as in our definitions: \emph{Naive pooling} (Def.~\ref{def:naive_pooling}), \emph{Uniform subsampling} (Def.~\ref{def:subsampling}), and \emph{Matching} (Def.~\ref{def:matching}). For matching we use $M(\cdot)=\|\cdot\|_2$, initialize the centroid with the sample-wise coordinate wise median, and iterate until convergence $\|\mathbf{c}^{(t+1)}-\mathbf{c}^{(t)}\|_2<10^{-4}$, with fixed radii $\tau\in\{1.0,1.1,1.2\}$).

\begin{figure}[t]
    \centering
    \includegraphics[width=0.7\textwidth]{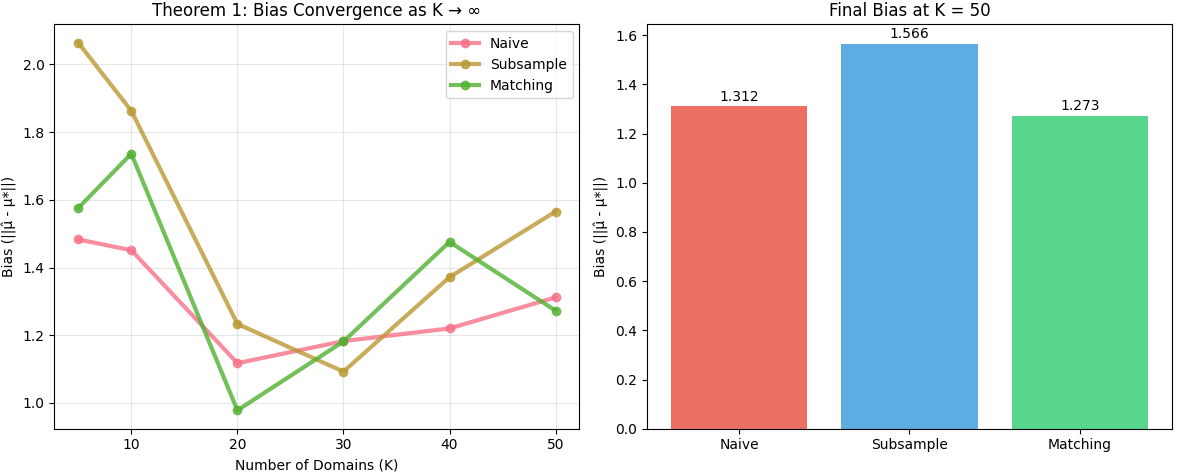}
    \caption{Asymptotic behavior (Theorem~\ref{thm:asymptotic_behavior}). Left: bias $\epsilon_K=\|\bar{\boldsymbol{\mu}}_K-\boldsymbol{\mu}_*\|_2$ versus number of domains $K\in\{5,10,20,30,40,50\}$. Right: final bias at $K=50$. Matching achieves the smallest bias as $K$ grows, in line with its convergence to $\mathcal{N}(\boldsymbol{\mu}_*,\sigma^2\mathbf{I}_d)$, while other strategies retain inter-domain variance.}
    \label{fig:synthetic_exp_1}
\end{figure}

\paragraph{Asymptotic regime ($K\to\infty$).}
For $K\in\{5,10,20,30,40,50\}$ under an asymmetric $\mathscr{D}_{\boldsymbol{\mu}}$ (Fig.~\ref{fig:synthetic_intro}), we sample $n=150$ points per domain and set $\tau=1.2$ for matching. Figure~\ref{fig:synthetic_exp_1} shows that the bias $\epsilon_K=\|\bar{\boldsymbol{\mu}}_K-\boldsymbol{\mu}_*\|_2$ decreases fastest for matching and is lowest at $K=50$, confirming Theorem~\ref{thm:asymptotic_behavior}: matching filters out inter-domain heterogeneity (converging to $\mathcal{N}(\boldsymbol{\mu}_*,\sigma^2\mathbf{I}_d)$), whereas naive pooling and subsampling converge to limits with added covariance $\boldsymbol{\Sigma}_{\boldsymbol{\mu}}$.

\begin{figure}[t]
    \centering
    \includegraphics[width=0.7\textwidth]{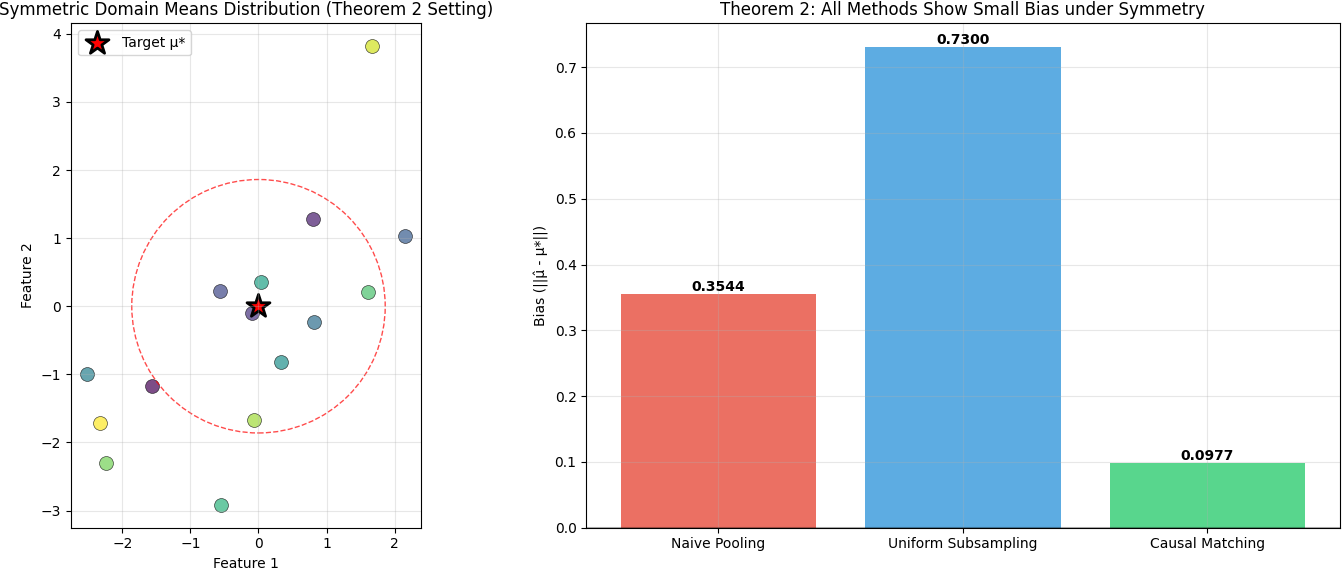}
    \caption{Finite-$K$ symmetry (Theorem~\ref{thm:symmetric_convergence}). Left: symmetrically oriented domain means around $\boldsymbol{\mu}_*$. Right: all methods exhibit negligible bias at $K=15$.}
    \label{fig:synthetic_exp_2}
\end{figure}

\paragraph{Finite-$K$ with symmetric meta-distribution.}
With $K=15$ and a symmetric $\mathscr{D}_{\boldsymbol{\mu}}$ (spread parameter $\gamma=1.5$), we again draw $n=100$ points per domain and set $\tau=1.0$. Figure~\ref{fig:synthetic_exp_2} confirms Theorem~\ref{thm:symmetric_convergence}: all three strategies are (nearly) unbiased for finite $K$ under symmetry, as the domain means balance around $\boldsymbol{\mu}_*$.

\begin{figure}[t]
    \centering
    \includegraphics[width=0.7\textwidth]{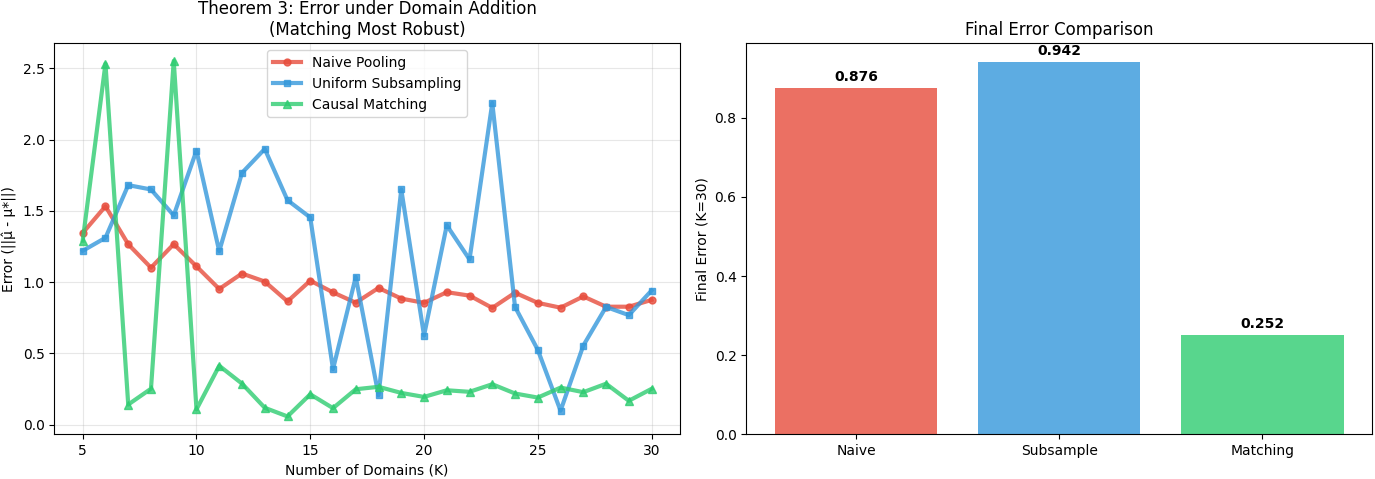}
    \caption{Finite-sample robustness under domain addition (Theorem~\ref{thm:finite_robustness}). Top-left: error trajectories as $K$ increases from $5$ to $30$ with every third domain an outlier ($\|\boldsymbol{\mu}_k-\boldsymbol{\mu}_*\|_2\approx 2.5$). Top-right: final errors at $K=30$. Bottom-left: domain means at the final $K$. Bottom-right: per-step error changes $\Delta\epsilon_K$. Matching is the most stable: bounded, shrinking perturbations and frequent non-deterioration steps.}
    \label{fig:synthetic_exp_3}
\end{figure}

\paragraph{Finite-$K$ with asymmetric addition (one-step robustness).}
Starting from $K{=}5$ and growing to $K{=}30$, we add a new domain at each step; every third domain is a harmful outlier with $\|\boldsymbol{\mu}_k-\boldsymbol{\mu}_*\|_2\approx 2.5$. We use $n=100$ samples per domain and set $\tau=1.1$ for matching. Figure~\ref{fig:synthetic_exp_3} shows: (i) \emph{naive pooling} exhibits occasional $\Theta(1)$ jumps when an outlier arrives; (ii) \emph{subsampling} blunts spikes but still incurs $\Omega(1/m)$ worst-case steps; and (iii) \emph{matching} yields the most stable curve and the smallest final error. The $\Delta\epsilon_K$ panel (bottom-right) highlights that matching’s per-step changes are tightly controlled and often non-positive, aligning with the high-probability non-deterioration guaranteed by Corollary~\ref{cor:optimal_tau}.

\noindent\textit{Takeaway.} Across all regimes, the empirical trends match the theory: matching removes inter-domain variance in the limit, is unbiased under symmetry, and crucially in finite, asymmetric settings keeps per-step error changes bounded by $O(1/|S_K|)$ with frequent non-deterioration, whereas naive pooling and subsampling can suffer persistent spikes. \footnote{Classical robust estimation methods, such as the tail-based filtering of Diakonikolas and Kane~\cite{diakonikolas2023algorithmic} (Chapter 2), rely on sub-Gaussian concentration inequalities to exclude outliers whose deviation exceeds a probabilistic threshold. This procedure assumes a Euclidean sample space where distance and variance share a direct linear relationship. However, on non-Euclidean manifolds, these
concentration bounds no longer hold, since distances do not scale linearly with
variance and curvature alters the shape of confidence regions. Consequently,
tail-based filtering can either reject valid domains located along high-curvature
directions or retain distant ones that appear close under Euclidean metrics,
thereby breaking domain inclusion consistency. Our $\tau$-based selection operates purely in geometric space: a domain is included if its geodesic distance from the centroid lies within a fixed radius $\tau$. Because compact manifolds are bi-Lipschitz equivalent to Euclidean spaces locally, this deterministic threshold guarantees bounded inclusion even when global sub-Gaussian assumptions fail. Hence, our method preserves inclusion–exclusion consistency under curvature.}

\subsection{Zero-shot anomaly detection}
\label{sec:zero_shot_supplementary}

\subsubsection{Related works}
\label{sec:related_works}

\textbf{Propensity score matching techniques:} Classical \emph{propensity score matching} (PSM)~\cite{rosenbaum1983central} and its extensions 
(e.g., inverse propensity weighting, covariate balancing, kernel balancing) are widely used in observational studies to reduce selection bias. These methods rely on the assumption of \emph{overlap} between treated and control distributions, ensuring that propensity scores do not concentrate near $0$ or $1$. When overlap holds, balancing based on estimated scores yields consistent treatment effect estimates. However, in \emph{multimodal or high-dimensional feature spaces}, such as CLIP embeddings, these techniques often become unstable. First, distinct modality clusters (e.g., ChestXRay vs.\ BrainMRI) occupy disjoint regions of the hypersphere, violating overlap (positivity) and causing propensity weights to 
explode or degenerate. Second, high dimensional propensity estimation requires strong parametric assumptions, and misspecification in multimodal settings leads to extreme variance in weights. Third, traditional balancing aligns global distributions, which in hyperspherical embeddings may force 
artificial alignment across unrelated modes, distorting semantic structure. Recent works have attempted to stabilize causal balancing under such conditions, for example through representation learning with generalization bounds~\cite{shalit2017estimating}, or by introducing optimal transport weights for causal inference~\cite{dunipace2021optimal,gunsilius2021matching}. While these provide more flexible mechanisms, they remain global and do not explicitly respect 
local manifold geometry. By contrast, our framework performs \emph{centroid-based geodesic matching} directly in hyperspherical feature space. Instead of relying on global balancing, it adaptively selects samples within geometry-consistent neighborhoods around centroids. Furthermore, we introduce variance-aware channel attention (VACA), which distinguishes class-discriminative from modality-driven variance. This yields cleaner neighborhoods for matching and avoids the instabilities that plague PSM and OT-based balancing in multimodal regimes. 

\textbf{Difference from Continual Learning:} While our experimental protocol superficially resembles \emph{continual learning} (CL), the 
fundamental challenges differ in scope and difficulty. In standard CL, the learner receives a 
sequence of tasks or domains and must avoid \emph{catastrophic forgetting} of previously acquired 
knowledge~\cite{li2025memory}. Techniques such as parameter-efficient fine-tuning, replay 
buffers, or regularization are designed to preserve past performance while adapting to new tasks. 
Crucially, CL typically assumes task boundaries are well-defined and that past and present tasks 
share at least partial feature overlap. Our setup is strictly more challenging. We consider the \emph{zero-shot anomaly detection} problem 
under domain addition, where each new domain may introduce fundamentally different modalities with 
disjoint embeddings on the hypersphere (e.g., ChestXRay vs.\ BrainMRI). Unlike CL, there is no 
guarantee of feature overlap (positivity may be violated), and performance must generalize to 
entirely unseen modalities without labeled adaptation. In this sense, our problem can be viewed as 
a \emph{hardest-case subset of continual learning}, where not only must past knowledge be retained, 
but the learner must construct geometry-consistent neighborhoods that ensure transfer across 
disconnected domains. Matching with variance-aware geodesic centroids directly addresses this 
difficulty, going beyond CL methods like BiLORA~\cite{zhu2025bilora} that focus only on preventing 
forgetting but do not resolve modality-induced separations.

\textbf{Existing architectures for zero-shot anomaly detection:}  
Medical anomaly detection faces extreme heterogeneity and cross-domain asymmetry, making robust zero-shot generalization a fundamental challenge. Traditional anomaly detection approaches, such as reconstruction-based methods \cite{deng2022anomaly,roth2022towards,salehi2021multiresolution}, often suffer from ambiguous decision boundaries \cite{bergmann2019mvtec}. Patch-based frameworks like PatchCore \cite{roth2022towards} and UniAD \cite{you2022unified} achieve strong performance in single-domain settings but lack adaptability to unseen domains. One-class classification methods \cite{bao2024bmad,jiang2023multi} also struggle to generalize across modalities, limiting their practical impact in heterogeneous medical environments. Recently, vision-language models (VLMs) have emerged as a powerful alternative. Methods such as AnomalyCLIP \cite{zhou2023anomalyclip} and AdaCLIP \cite{cao2024adaclip} leverage CLIP embeddings and lightweight adapters to enable anomaly detection without per-domain retraining. These approaches show promise but remain challenged by modality-specific feature clustering, as anomalies often differ structurally across modalities. MVFA-AD \cite{huang2024adapting} further adapts CLIP to medical images via modality-specific adapters, improving robustness but still susceptible to performance degradation when confounding domains are introduced.  

We adopt and extend this VLM-based direction for three principled reasons: (1) \textbf{Semantic grounding}, where VLMs encode pathological concepts beyond modality-specific appearances; (2) \textbf{Transferable representations}, as large-scale pretraining provides resilience against distribution shifts; and (3) \textbf{Zero-shot capability}, enabling generalization to unseen modalities without retraining. Our framework builds directly on AnomalyCLIP \cite{zhou2023anomalyclip}, incorporating both its \emph{task-agnostic prompts} and \emph{adapter modules}, but enhances it with a theoretically grounded geodesic matching mechanism that mitigates modality-induced clustering.  

Finally, since our evaluation protocol resembles continual domain addition, we also benchmark against BiLORA \cite{zhu2025bilora}, a parameter-efficient fine-tuning (PEFT) framework for continual learning that uses adapter-based LoRA to prevent catastrophic forgetting. Unlike BiLORA \cite{zhu2025bilora} and related baselines, our method explicitly prevents performance deterioration under domain addition, achieving robustness consistent with our theoretical guarantees. 

\subsubsection{Modality Specific Clusters}
\label{sec:modality_clusters}

Following the findings and technique introduced by \cite{liang2022mind}, we find the cosine similarites among the images of the datasets (modalities).  500 images for each dataset was chosen and the image features from the fourth adapter layer (the deepest adapter layer) of the CLIP image encoder for calculating the intra and inter domain cosine similarities. Table \ref{tab:dataset_similarity} reports the pairwise cosine similarity between dataset-level centroids in the CLIP embedding space. The diagonal entries are consistently dominant, indicating that each modality occupies a relatively narrow region of the hyperspherical feature space. In contrast, the off-diagonal similarities are substantially lower, reflecting limited overlap across modalities. This directly supports our observation of \emph{modality-specific clustering}: embeddings from the same modality are tightly grouped, whereas embeddings from different modalities are separated by large angular gaps. Consequently, zero-shot generalization becomes challenging, as models must traverse long distances across disjoint clusters to adapt from seen to unseen domains.

\begin{table}[t]
\centering
\caption{Pairwise cosine similarity matrix between dataset centroids in the CLIP feature space. High diagonal values highlight narrow intra-modality clustering, while low off-diagonal values reflect inter-modality separation.}
\label{tab:dataset_similarity}
\begin{tabular}{l c c c c c c}
\toprule
 & BrainMRI & LiverCT & OCT2017 & RESC & ChestXray & HIS \\
\midrule
BrainMRI            & 0.820 & 0.240 & 0.540 & 0.431 & 0.479 & 0.468 \\
LiverCT            & 0.240 & 0.371 & 0.201 & 0.110 & 0.290 & 0.208 \\
OCT2017       & 0.540 & 0.201 & 0.820 & 0.645 & 0.785 & 0.816 \\
RESC       & 0.431 & 0.110 & 0.645 & 0.594 & 0.591 & 0.651 \\
ChestXray            & 0.479 & 0.290 & 0.785 & 0.591 & 0.908 & 0.840 \\
HIS   & 0.468 & 0.208 & 0.816 & 0.651 & 0.840 & 0.898 \\
\bottomrule
\end{tabular}
\end{table}

\subsection{Comparison of the Class Discriminative Power of the image feature dimensions}
\label{sec:variance_supplementary}
From Figure~\ref{fig:variance_plot} (RetinaOCT2017), we observe that only a subset of embedding dimensions contributes substantially to the class-specific variance (normal vs.\ anomaly), consistent with our formulation in Section~\ref{sec:method}. These dimensions capture discriminative structure aligned with pathology, while the remaining dimensions encode modality-specific variance and are responsible for forming modality-specific clusters on the hypersphere (see also Section~\ref{sec:variance_attention}). This validates the need for variance-aware channel attention (VACA), which explicitly reweights channels by their discriminative contribution $\Delta$ and aggregates them through the variance surrogate $\mathsf{Var}_{\mathrm{disc}}$.

To provide intuition, consider a simplified 3D sphere. If the embeddings of two modalities lie along orthogonal axes, e.g., $(1,0,0)$, $(0,1,0)$, and $(0,0,1)$, then each point is supported by a distinct single dimension. Although all points lie on the same sphere, they are maximally separated in terms of discriminative variance and cannot overlap meaningfully. In contrast, if embeddings are distributed across dimensions, such as $(\tfrac{1}{2},\tfrac{1}{2},\tfrac{1}{2})$ and $(\tfrac{1}{2},\tfrac{1}{4},\tfrac{1}{3})$, then both share support across multiple dimensions, leading to closer angular separation. The first case illustrates modality-dominated variance (dimensions act independently), while the second corresponds to discriminative variance (dimensions contribute jointly).

VACA operationalizes this principle by amplifying channels that show consistent discriminative alignment (like shared multi-dimensional support) and downweighting those that act as independent modality indicators. Geometrically, this shrinks modality-induced angular gaps on $\mathbb{S}^{d-1}$, thereby making geodesic neighborhoods more semantically consistent. Empirically, applying VACA increases the effective discriminative energy across all embedding dimensions, enhancing robustness in the zero-shot anomaly detection setting.

\begin{figure}[t]
    \centering
    \includegraphics[width=0.9\textwidth]{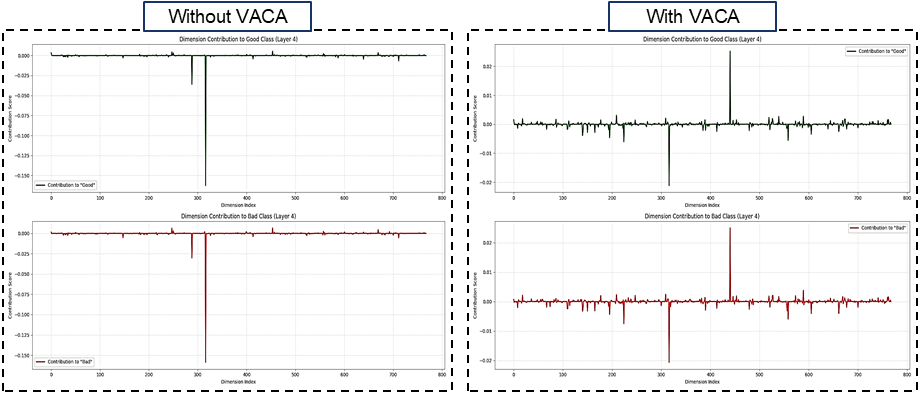}
    \caption{Variance decomposition of image embeddings for the last adapter layer across modalities for RetinaOCT2017. A subset of channels carries class-specific variance (normal vs.\ anomaly), while the majority encode modality-specific variance. This induces hyperspherical clustering and motivates VACA (see Section~\ref{sec:variance_attention}), which reweights channels to amplify discriminative signal.}
    \label{fig:variance_plot}
\end{figure}

\subsubsection{Ablation Experiments}
\label{sec:ablation}
Figure~\ref{fig:ablation} illustrates the impact of the choice of matching metric $M$ on both alignment and downstream performance. Consistent with our geometric analysis in Section~\ref{sec:mind_the_gap}, geodesic distance $d_G$ outperforms Euclidean and cosine metrics across most settings. This aligns with the theoretical guarantees of Section~\ref{sec:geodesic_centroid}, where $d_G$ was shown to faithfully represent angular separations on the hypersphere, thereby preserving local neighborhoods and strengthening the positivity condition in Theorem~\ref{thm:finite_robustness}.  

Moreover, integrating Variance Aware Channel Attention (VACA, Section~\ref{sec:variance_attention}) with geodesic matching (CS\_GeodVar) yields further gains. VACA reduces modality-induced variance by upweighting class-discriminative channels, effectively tightening the semantic neighborhoods used in the matching step. This ensures that matched sets satisfy a stronger form of local positivity, as spurious modality-specific dimensions are attenuated. Empirically, this manifests as consistently higher DA scores, with AC and AS performance following the same upward trend.   For these ablation plots, we deliberately adopt the \emph{most difficult sequence of domain addition}, where domains are introduced in order of increasing angular distance—first the closest domain, then progressively more distant ones (see Table \ref{tab:dataset_similarity}). This design produces the hardest-case incremental generalization setting, where deterioration is most likely. The strong performance of $d_G$ and VACA under this protocol underscores the robustness of our approach. Together, these results validate our design: geodesic distance serves as the theoretically principled matching metric on $\mathbb{S}^{d-1}$, while VACA enhances the robustness of matching by ensuring that the effective feature space better aligns with the discriminative axes rather than modality-specific ones.

\subsubsection{Comparison with existing state-of-the-art methods}
\label{sec:sota}

Figure~\ref{fig:final} highlights the comparative performance of MVFA~\cite{huang2024adapting}, AnomalyCLIP~\cite{zhou2023anomalyclip}, BiLORA~\cite{zhu2025bilora}, and our approach under sequential domain addition for anomaly classification (AC), anomaly segmentation (AS), and domain alignment (DA). MVFA \cite{huang2024adapting}, which naively pools all domains, consistently produces only moderate DA scores (2.0–2.2), confirming our theoretical results that naive pooling amplifies distributional asymmetries and degrades generalization under domain shifts. AnomalyCLIP \cite{zhou2023anomalyclip}, built on text-agnostic prompts, fares worse—its DA scores swing erratically from 1.01 to 3.26—demonstrating the instability of heuristic prompt tuning when confronted with heterogeneous domains. BiLORA \cite{zhu2025bilora}, a continual learning framework with parameter-efficient fine-tuning (PEFT) designed to mitigate catastrophic forgetting, improves upon both MVFA and AnomalyCLIP. Yet, because it still relies on naive pooling across sequentially added domains, its DA values remain below the critical threshold of 4 in several cases (e.g., 3.145 for ChestXRay AC, 3.158 for Liver CT AC, 3.081 for RESC AC). This leaves BiLORA vulnerable to performance dips whenever the introduced domains are misaligned with prior ones—catastrophic forgetting may be mitigated, but confounding from domain misalignment is not. Furthermore, Table~\ref{tab:zero-shot} shows that our method achieves performance on par with, and often surpassing, existing state-of-the-art approaches for anomaly detection. Importantly, it attains this level of accuracy while simultaneously guaranteeing stability under sequential domain addition, ensuring no deterioration in performance as new domains are introduced for fine-tuning.

In contrast, our method achieves DA scores consistently at or above 4.0 for every dataset and task, including peaks of 4.0736 for Brain MRI detection (AC) and 4.0484 for Brain MRI segmentation (AS). Crossing this threshold is significant: DA $\geq 4$ implies that alignment is never compromised by domain addition, ensuring no risk of collapse or instability across heterogeneous sources. This robustness stems directly from our design. Variance-Aware Channel Attention (VACA) selectively amplifies discriminative channels while suppressing spurious modality-driven variance, and geodesic matching on the hypersphere adaptively filters samples so that only label-consistent features refine centroids. Together, these mechanisms enforce stable, geometry-consistent clustering and prevent centroids from drifting toward confounded or ambiguous features.

Theoretically, this aligns with the \textit{ignorability} and \textit{positivity} assumptions underpinning our framework. Unlike MVFA \cite{huang2024adapting}, AnomalyCLIP \cite{zhou2023anomalyclip}, and BiLORA \cite{zhu2025bilora}—which indiscriminately pool domains and violate ignorability by letting confounded features dominate—our method ensures that updates satisfy positivity and ignorability, which are the two fundamental assumptions of causality for an unbiased estimator. This principled integration of geometry and variance-aware filtering explains why our DA never dips below 4 and has a higher chance of increasing, providing provable stability under domain addition and superior practical robustness.

\textit{Note: All experiments were conducted using the hardest sequence of domain addition, where the domain closest to the test domain (as determined from Table~\ref{tab:dataset_similarity}) was introduced first for fine-tuning, followed by the second closest, and so on. This ordering makes performance gains progressively harder with each added domain. To further evaluate robustness, we also tested two additional sequences for both our method (matching) and BiLORA~\cite{zhu2025bilora} (naive pooling): a) Reverse sequence (easiest case): The order was inverted, starting with the farthest domain and moving toward the test domain. This maximizes the chance of performance improvement at later stages, since domains closer to the test set are utilized later, generally yielding higher DA scores. b) Arbitrary sequence: A randomly chosen order of domain addition was used to test invariance to sequencing. The purpose of these tests was to examine exchangeability (Corollary~\ref{cor:exchangeability}), i.e., whether the final performance after all domains are introduced remains the same regardless of order. Since exchangeability is a necessary assumption for pooling methods to ensure ignorability, verifying this empirically is important. We observed that the final performance for both BiLORA \cite{zhu2025bilora} and our proposed method differed by at most $\pm$1.06 AUC across sequences. This confirms that exchangeability holds in practice.}

\subsubsection{Overall Framework}
\label{sec:algorithm}
Algorithm~\ref{alg:main_method} summarizes our pipeline integrating variance-aware channel attention (VACA), intrinsic (geodesic) matching on the unit hypersphere, and multi-task objectives. Given patch embeddings $P\!\in\!\mathbb{R}^{B\times N\times D}$ and class text embeddings $T\!\in\!\mathbb{R}^{D\times 2}$, we first adapt $P$ with task-specific lightweight adapters for \emph{segmentation} and \emph{detection} to disentangle subspaces. VACA then emphasizes class-discriminative channels while suppressing modality-driven variance via learned gains derived from channel-wise semantic contrasts, yielding reweighted features $\widetilde{P}_{\text{seg}}$ (used for matching) and $\widetilde{P}_{\text{det}}$ (used for image-level scoring). 

Next, we perform \emph{geodesic matching} on the unit hypersphere with \emph{normal} and \emph{anomaly} centroids denoted by $\mathbf{c}^{+}$ and $\mathbf{c}^{-}$, respectively. For a image feature $f_j$ corresponding to the image sample $x_j$, we compute $d^{+} \!=\! d_G(f_j,\mathbf{c}^{+})$ and $d^{-} \!=\! d_G(f_j,\mathbf{c}^{-})$, where $d_G$ is the geodesic distance. We employ an \textbf{adaptive $\tau$}: $\tau(f_j) \;=\; \min\!\big(d^{+},\, d^{-}\big)$.

A sample updates \emph{only} its label-consistent centroid when it is strictly closer than the other, i.e., for normal samples update if $d^{+}\!<\!d^{-}$ and for anomaly samples update if $d^{-}\!<\!d^{+}$; otherwise we \emph{skip} the sample. This label-consistent, threshold-free rule prevents ambiguous pulls and yields stable prototype refinement as centroids evolve via exponential moving averages (EMA). 

In parallel, pixel-level losses (Dice, Focal) and image-level BCE align adapted visual features with text embeddings. To prevent mode collapse and guarantee identifiability, we push $\mathbf{c}^{+}$ and $\mathbf{c}^{-}$ apart with the \emph{geodesic} loss 
\[
\mathcal{L}_{\text{intra}}
\;=\;
\frac{1}{|\mathcal{S}^{+}|}\!\sum_{f\in\mathcal{S}^{+}} d_G(f,\mathbf{c}^{+})^{2}
\;+\;
\frac{1}{|\mathcal{S}^{-}|}\!\sum_{f\in\mathcal{S}^{-}} d_G(f,\mathbf{c}^{-})^{2},
\quad
\mathcal{L}_{\text{inter}}
\;=\;
-\, d_G(\mathbf{c}^{+}, \mathbf{c}^{-}),
\]
and
\[
\mathcal{L}_{\text{geo}}
\;=\;
\lambda_{1}\,\mathcal{L}_{\text{intra}}
\;+\;
\lambda_{2}\,\mathcal{L}_{\text{inter}},
\]
with $\lambda_1,\lambda_2\!>\!0$ balancing compactness and separation (consistent with Sec.~\ref{sec:ablation} discussion). The full objective is
\[
\mathcal{L}_{\text{total}}
\;=\;
\underbrace{\beta_{1}\,\mathcal{L}_{\text{Dice}}+\beta_{2}\,\mathcal{L}_{\text{Focal}}}_{\mathcal{L}_{\text{seg}}}
\;+\;
\underbrace{\beta_{3}\,\mathcal{L}_{\text{BCE}}}_{\mathcal{L}_{\text{det}}}
\;+\;
\underbrace{\mathcal{L}_{\text{geo}}}_{\text{centroids}}
\;+\;
\underbrace{\mathcal{L}_{\text{var}}}_{\text{VACA reg.}},
\]
where $\beta_{1},\beta_{2},\beta_{3}\!>\!0$ weight the task losses, and $\mathcal{L}_{\text{var}}$ is the VACA variance regularizer. $\mathcal{L}_{\text{var}}$ is only used to update the parameters of $MLP_\theta$. All hyperparameters $(\lambda_{1},\lambda_{2},\beta_{1..3},\lambda_{\text{var}})$ are selected via extensive experimentation per modality.  

\begin{algorithm}[t]
\caption{Variance-Aware Geodesic Matching for Zero-Shot Medical Anomaly Detection}
\label{alg:main_method}
\begin{algorithmic}[1]
\Require Patch embeddings $P \in \mathbb{R}^{B \times N \times D}$; class embeddings $T \in \mathbb{R}^{D \times 2}$
\Require Initial centroids $\mathbf{c}^{+}_{(0)}, \mathbf{c}^{-}_{(0)} \in \mathbb{R}^D$ (unit-norm or zero-init then normalized)
\Require EMA factor $\alpha \in (0,1)$; channel amplification $\gamma > 0$
\Require Loss weights $\beta_{1},\beta_{2},\beta_{3}\!>\!0$; $\lambda_{1},\lambda_{2}\!>\!0$; variance weight $\lambda_{\text{var}}\!>\!0$
\State \textbf{Initialize:} $\mathbf{c}^{+}\!\leftarrow\!\mathbf{c}^{+}_{(0)}$, $\mathbf{c}^{-}\!\leftarrow\!\mathbf{c}^{-}_{(0)}$, $\mathcal{S}^{+}\!\leftarrow\!\emptyset$, $\mathcal{S}^{-}\!\leftarrow\!\emptyset$

\For{each batch $(P, T, \mathcal{Y}_{\text{seg}}, \mathcal{Y}_{\text{det}})$}
    \State \textbf{(1) Task-Specific Adapters and Normalization}
    \State $P_{\text{seg}} \gets \mathrm{Adapter}_{\text{seg}}(P)$;\quad $P_{\text{det}} \gets \mathrm{Adapter}_{\text{det}}(P)$
    \State $P_{\text{norm}} \gets \ell_2\text{-normalize}(P_{\text{seg}})$;\quad $T_{\text{norm}} \gets \ell_2\text{-normalize}(T)$
    
    \State \textbf{(2) Variance-Aware Channel Attention (VACA)}
    \State $\tilde{P} \!\gets\! \text{broadcast}(P_{\text{norm}}) \in \mathbb{R}^{B\times N\times 1\times D}$;\;
           $\tilde{T} \!\gets\! \text{broadcast}(T_{\text{norm}}) \in \mathbb{R}^{1\times 1\times D\times 2}$
    \State $\mathcal{C} \!\gets\! \tilde{P}\odot\tilde{T}$;\quad $\overline{\mathcal{C}} \!\gets\! \frac{1}{BN}\sum_{b,n}\mathcal{C}_{b,n,:,:}\in\mathbb{R}^{D\times 2}$
    \State $\Delta_d \gets \big|\overline{\mathcal{C}}_{d,\text{normal}} - \overline{\mathcal{C}}_{d,\text{anomaly}}\big|$ \Comment{channel discriminability}
    \State $a \gets \mathrm{Softplus}(\mathrm{MLP}_\theta(\Delta))$;\quad $w \gets 1 + \gamma a$
    \State $\widetilde{P}_{\text{seg}} \gets P_{\text{norm}} \odot w$;\quad $\widetilde{P}_{\text{det}} \gets P_{\text{det}}$
    \State \textbf{VACA variance regularizer:}\;
           $\mathsf{Var}_{\mathrm{disc}} \gets \mathrm{Var}(\Delta \odot w)$;\;
           $\mathcal{L}_{\text{var}} \gets \lambda_{\text{var}}\cdot(\mathsf{Var}_{\text{seg}} + \mathsf{Var}_{\text{det}})$
    
    \State \textbf{(3) Geodesic Matching with Adaptive Gate}
    \For{each feature $f_j$ in $\widetilde{P}_{\text{seg}}$ with label $y\!\in\!\{\text{normal},\text{anomaly}\}$}
        \State $d^{+} \gets d_G(f_j,\mathbf{c}^{+}) = \arccos\!\big(\langle f_j,\mathbf{c}^{+}\rangle\big)$;\quad
               $d^{-} \gets d_G(f_j,\mathbf{c}^{-}) = \arccos\!\big(\langle f_j,\mathbf{c}^{-}\rangle\big)$
        \State $\tau(f_j) \gets \min(d^{+}, d^{-})$ \Comment{adaptive, sample-specific}
        \If{$y=\text{normal}$ \textbf{and} $d^{+} < d^{-}$}
            \State $\mathcal{S}^{+} \gets \mathcal{S}^{+} \cup \{f_j\}$;\;
                   $\tilde{\mathbf{c}}^{+} \gets \alpha\,\mathbf{c}^{+} + (1-\alpha)f_j$;\;
                   $\mathbf{c}^{+} \gets \tilde{\mathbf{c}}^{+}/\|\tilde{\mathbf{c}}^{+}\|_2$
        \ElsIf{$y=\text{anomaly}$ \textbf{and} $d^{-} < d^{+}$}
            \State $\mathcal{S}^{-} \gets \mathcal{S}^{-} \cup \{f_j\}$;\;
                   $\tilde{\mathbf{c}}^{-} \gets \alpha\,\mathbf{c}^{-} + (1-\alpha)f_j$;\;
                   $\mathbf{c}^{-} \gets \tilde{\mathbf{c}}^{-}/\|\tilde{\mathbf{c}}^{-}\|_2$
        \Else
            \State \textbf{skip} \Comment{ambiguous sample; no centroid update}
        \EndIf
    \EndFor
    
    \State \textbf{(4) Losses}
    \State \emph{Segmentation:}\;
           $\hat{\mathcal{Y}}_{\text{seg}} \gets \mathrm{softmax}(\widetilde{P}_{\text{seg}} T_{\text{norm}}^\top)$;\;
           $\mathcal{L}_{\text{Dice}} \gets 1-\mathrm{Dice}(\hat{\mathcal{Y}}_{\text{seg}}, \mathcal{Y}_{\text{seg}})$;\;
           $\mathcal{L}_{\text{Focal}} \gets \mathrm{FocalLoss}(\hat{\mathcal{Y}}_{\text{seg}}, \mathcal{Y}_{\text{seg}})$;\;
           $\mathcal{L}_{\text{seg}} \gets \beta_{1}\mathcal{L}_{\text{Dice}} + \beta_{2}\mathcal{L}_{\text{Focal}}$
    \State \emph{Detection:}\;
           $\mathbf{s}_{\text{det}} \gets \max_{\text{patch}}(\widetilde{P}_{\text{det}} T_{\text{norm}}^\top)$;\;
           $\mathcal{L}_{\text{det}} \gets \beta_{3}\,\mathrm{BCE}(\sigma(\mathbf{s}_{\text{det}}), \mathcal{Y}_{\text{det}})$
    \State \emph{Geodesic (centroids):}\;
           $\mathcal{L}_{\text{intra}} \gets \frac{1}{|\mathcal{S}^{+}|}\sum_{f\in\mathcal{S}^{+}} d_G(f,\mathbf{c}^{+})^{2}
           + \frac{1}{|\mathcal{S}^{-}|}\sum_{f\in\mathcal{S}^{-}} d_G(f,\mathbf{c}^{-})^{2}$;\;
           $\mathcal{L}_{\text{inter}} \gets -\, d_G(\mathbf{c}^{+},\mathbf{c}^{-})$
    \State \phantom{\emph{Geodesic (centroids):}}\;
           $\mathcal{L}_{\text{geo}} \gets \lambda_{1}\mathcal{L}_{\text{intra}} + \lambda_{2}\mathcal{L}_{\text{inter}}$
    \State \textbf{Total:}\;
           $\mathcal{L}_{\text{total}} \gets \mathcal{L}_{\text{seg}} + \mathcal{L}_{\text{det}} + \mathcal{L}_{\text{geo}} + \mathcal{L}_{\text{var}}$
\EndFor
\end{algorithmic}
\end{algorithm}

\end{document}